\documentclass{ifcolog}
\usepackage{amsmath}
\usepackage{amssymb}
\usepackage{color}
\usepackage{named}

\long\def\BOC#1\EOC{\message{(Commented text )}}
\long\def\BOCC#1\EOCC{\message{(Commented text )}}
\long\def\BOCCC#1\EOCCC{\message{(Commented text )}}
\long\def\optional#1{\empty}
\def\ar{\leftarrow}
\def\bi{\begin{itemize}}
\def\ei{\end{itemize}}
\def\beq{\begin{equation}}
\def\eeq#1{\label{#1}\end{equation}}
\def\ba{\begin{array}}
\def\ea{\end{array}}
\def\i#1{\hbox{\it #1\/}}
\def\mi#1{\mathit{#1}}
\def\sm{\hbox{\rm SM}}
\def\sneg{\sim\!\!}
\def\ar{\leftarrow}
\def\rar{\rightarrow}

\def\false{\hbox{\sc false}}
\def\true{\hbox{\sc true}}

\def\i#1{\hbox{\itshape #1\/}}

\def\qed{\quad \vrule height7.5pt width4.17pt depth0pt \medskip}
     
\def\Lrar{\Leftrightarrow}
\DeclareSymbolFont{AMSa}{U}{msa}{m}{n}
\DeclareMathSymbol{\square}{\mathord}{AMSa}{"03}

\def\mu#1{\mathit{\underline{#1}}}

\def\fand{\otimes}
\def\for{\oplus}
\def\fneg{\neg}
\def\frar{\rar}

\def\bi{\begin{itemize}}
\def\ei{\end{itemize}}
\def\ii{\medskip\item}

\def\u{\upsilon}

\def\V{{\sf V}}
\def\I{{\sf I}}

\newtheorem{prop}{Proposition}
\newtheorem{thm}{Theorem}
\newtheorem{cor}{Corollary}
\newtheorem{definition}{Definition}
\newtheorem{lemma}{Lemma} 
\newtheorem{example}{Example}

\title{Fuzzy Propositional Formulas under the Stable Model Semantics}

\titlerunning{Fuzzy Propositional Formulas under the Stable Model Semantics}

\addauthor[joolee@asu.edu]
  {Joohyung Lee}
  {School of Computing, Informatics, and Decision Systems Engineering \\
  Arizona State University, Tempe, USA}

\addauthor[ywang485@asu.edu]
  {Yi Wang}
  {School of Computing, Informatics, and Decision Systems Engineering \\
  Arizona State University, Tempe, USA}

\titlethanks{We are grateful to Joseph Babb, Michael Bartholomew, Enrico Marchioni, and the anonymous referees for their useful comments and discussions related to this paper. We thank Mario Alviano for his help in running the system {\sc fasp2smt}.
This work was partially supported by the National Science Foundation under Grants IIS-1319794 and IIS-1526301, and by the South Korea IT R\&D program MKE/KIAT
2010-TD-300404-001.}

\authorrunning{Joohyung Lee and Yi Wang} 

\begin{document}

\maketitle


\begin{abstract}
We define a stable model semantics for fuzzy propositional formulas, which generalizes both fuzzy propositional logic and the stable model semantics of classical propositional formulas. The syntax of the language is the same as the syntax of fuzzy propositional logic, but its semantics distinguishes stable models from non-stable models. The generality of the language allows for highly configurable nonmonotonic reasoning for dynamic domains involving graded truth degrees. We show that several properties of Boolean stable models are naturally extended to this many-valued setting, and discuss how it is related to other approaches to combining fuzzy logic and the stable model semantics. 
\end{abstract}

\section{Introduction} \label{sec:intro}\optional{sec:intro}

Answer Set Programming (ASP) \cite{lif08} is a widely applied declarative programming paradigm for the design and implementation of knowledge-intensive applications. One of the attractive features of ASP is its capability to represent  the nonmonotonic aspect of knowledge. However, its mathematical basis, the stable model semantics, is restricted to Boolean values and is too rigid to represent imprecise and vague information. Fuzzy logic, as a form of many-valued logic, can handle such information by interpreting propositions with graded truth degrees, often as real numbers in the interval $[0,1]$.
The availability of various fuzzy operators gives the user great
flexibility in combining the truth degrees. However, the
semantics of fuzzy logic is monotonic and is not flexible enough to
handle default reasoning as in answer set programming.

Both the stable model semantics and fuzzy logic are generalizations of classical propositional logic in different ways. While they do not subsume each other, it is clear that many real-world problems require both their strengths. This has led to the body of work on combining fuzzy logic and the stable model semantics, known as fuzzy answer set programming (FASP) (e.g., \cite{lukasiewicz06fuzzy,janssen12reducing,vojtas01fuzzy,damasio01monotonic,medina01multi-adjoint,damasio01antitonic,nieuwenborgh07anintroduction,madrid08towards}).
However, most works consider simple rule forms and do not allow logical connectives nested arbitrarily as in fuzzy logic. An exception is fuzzy equilibrium logic \cite{schockaert12fuzzy}, which applies to arbitrary propositional formulas even including strong negation. However, its definition is highly complex.

Unlike the existing works on fuzzy answer set semantics, in this paper, we extend the general stable model semantics from~\cite{fer05,ferraris11stable} to fuzzy propositional formulas. The syntax of the language is the same as the syntax of fuzzy propositional logic. The semantics, on the other hand, defines a condition which distinguishes {\sl stable} models from non-stable models. The language is a proper generalization of both fuzzy propositional logic and classical propositional formulas under the stable model semantics, and turns out to be essentially equivalent to fuzzy equilibrium logic, but is much simpler. Unlike the interval-based semantics in fuzzy equilibrium logic, the proposed semantics is based on the notions of a reduct and a minimal model, familiar from the usual way stable models are defined, and thus provides a simpler, alternative characterization of fuzzy equilibrium logic.
In fact, in the absence of strong negation, a fuzzy equilibrium model always assigns an interval of the form $[v, 1]$ to each atom, which can be simply identified with a single value $v$ under our stable model semantics. Further, we show that strong negation can be eliminated from a formula in favor of new atoms, extending the well-known result in answer set programming. So our simple semantics fully captures fuzzy equilibrium logic.

Also, there is a significant body of work based on the general stable model semantics, such as the splitting theorem~\cite{ferraris09symmetric}, the theorem on loop formulas~\cite{fer06}, and the concept of aggregates~\cite{lee08}. The simplicity of our semantics would allow for easily extending those results to the many-valued setting, as can be seen from some examples in this paper.

Another contribution of this paper is to show how reasoning about dynamic systems in ASP can be extended to fuzzy ASP.  It is well known that actions and their effects on the states of the world can be conveniently represented by answer set programs \cite{lif99b,lif02}. On the contrary, to the best of our knowledge, the work on fuzzy ASP has not addressed how the result can be extended to many-valued setting, and most applications discussed so far are limited to static domains only.
As a motivating example, consider modeling dynamics of {\em trust} in social networks. People trust each other in different degrees under some normal assumptions. If person $A$ trusts person $B$, then $A$ tends to trust person $C$ whom $B$ trusts, to a degree which is positively correlated to the degree to which $A$ trusts $B$ and the degree to which $B$ trusts $C$. By default, the trust degrees would not change, but may decrease when some conflict arises between the people. 
Modeling such a domain requires expressing defaults involving fuzzy truth values. 
We demonstrate how the proposed language can achieve this by taking advantage of its generality over the existing approaches to fuzzy ASP.
Thus the generalization is not simply a pure theoretical pursuit but has practical uses in the convenient modeling of defaults involving fuzzy truth values in dynamic domains. 

The paper is organized as follows. Section~\ref{sec:prelim} reviews the syntax and the semantics of fuzzy propositional logic, as well as the stable model semantics of classical propositional formulas. Section~\ref{sec:definition} presents the main definition of a stable model of a fuzzy propositional formula along with several examples. 
Section~\ref{sec:gen-alt} presents a generalized definition based on partial degrees of satisfaction and its reduction to the special case, as well as an alternative, second-order logic style definition. 
Section~\ref{sec:rel-bool} tells us how the Boolean stable model semantics can be viewed as a special case of the fuzzy stable model semantics, and Section~\ref{sec:rel-fuzzyasp} formally compares the fuzzy stable model semantics with normal FASP programs and fuzzy equilibrium logic. Section~\ref{sec:properties} shows that some well-known properties of the Boolean stable model semantics can be naturally extended to our fuzzy stable model semantics. Section~\ref{sec:related-work} discusses other related work, followed by the conclusion in Section~\ref{sec:conclusion}. The complete proofs are given in the appendix. 

This paper is a significantly extended version of the papers~\cite{lee14stable1,lee14stable}. Instead of the second-order logic style definition used there, we present a new, reduct-based definition as the main definition, which is simpler and more aligned with the standard definition of a stable model. 
Further, this paper shows that a generalization of the stable model semantics that allows partial degrees of satisfaction can be reduced to the version that allows only the two-valued concept of satisfaction.

\section{Preliminaries} \label{sec:prelim} \optional{sec:prelim}

\subsection{Review: Fuzzy Logic}  \label{ssec:review-fuzzy}
  \optional{ssec:review-fuzzy}

A fuzzy propositional signature~$\sigma$ is a set of symbols called {\em fuzzy atoms}. In addition, we assume the presence of a set ${\rm CONJ}$ of fuzzy conjunction symbols, a set ${\rm DISJ}$ of fuzzy disjunction symbols, a set ${\rm NEG}$ of fuzzy negation symbols, and a set ${\rm IMPL}$ of fuzzy implication symbols. 

A {\em fuzzy (propositional) formula} (of $\sigma$) is defined recursively as follows.
\begin{itemize}
\item every fuzzy atom $p\in\sigma$ is a fuzzy formula;

\item every numeric constant $c$, where $c$ is a real number in $[0, 1]$, is a fuzzy formula;

\item if $F$ is a fuzzy formula, then $\fneg F$ is a fuzzy formula, where $\fneg \in {\rm NEG}$;

\item if $F$ and $G$ are fuzzy formulas, then $F\fand G$, $F\for G$, 
  and $F\frar G$ are fuzzy formulas, where $\fand\in{\rm CONJ}$,
  $\for\in{\rm DISJ}$, and $\frar\ \in{\rm IMPL}$.
\end{itemize}

The models of a fuzzy formula are defined as follows \cite{hajek98mathematics}.
The {\em fuzzy truth values} are the real numbers in the range $[0,1]$. 
A \emph{fuzzy interpretation} $I$ of~$\sigma$ is a mapping from $\sigma$ into $[0, 1]$.

The fuzzy operators are functions mapping one or a pair of fuzzy truth values into a fuzzy truth value. Among the operators, $\fneg$ denotes a function from $[0, 1]$ into $[0, 1]$; $\fand$, $\for$, and $\frar$ denote functions from $[0, 1]\times[0, 1]$ into $[0, 1]$.
The actual mapping performed by each operator can be defined in many different ways, but all of them satisfy the following conditions, which imply that the operators are generalizations of the corresponding classical propositional connectives:\footnote{%
We say that a function $f$ of arity $n$ is \emph{increasing in its $i$-th argument} ($1\leq i\leq n$) if $f(arg_1,\dots, arg_i,\dots, arg_n)\leq f(arg_1,\dots,arg_i^\prime,\dots, arg_n)$ whenever $arg_i \leq arg_i^\prime$; 
$f$ is said to be \emph{increasing} if it is increasing in all its arguments. The definition of \emph{decreasing} is similarly defined.  
}

\begin{itemize}
\item a fuzzy negation $\fneg$ is decreasing, and satisfies $\neg(0) = 1$ and $\neg(1) = 0$;

\item a fuzzy conjunction $\fand$ is increasing, commutative, associative, and $\fand(1, x)=x$ for all $x \in [0, 1]$;

\item a fuzzy disjunction $\for$ is increasing, commutative, associative, and $\for(0, x)=x$ for all $x\in [0, 1]$;

\item a fuzzy implication $\frar$ is decreasing in its first argument and increasing in its second argument; and $\frar\!\!(1, x) = x$ and $\frar\!\!(0, 0) = 1$ for all $x\in[0, 1]$.
\end{itemize}

Figure~\ref{fig:operators} lists some specific fuzzy operators that we use in this paper.

\begin{figure}
\begin{center}
\begin{tabular}{| c | l | l |}
			\hline
			\textbf{Symbol} & \textbf{Name} & \textbf{Definition} \\
			\hline
			$\fand_l$ & \L ukasiewicz t-norm & $\fand_l(x, y)=max(x + y -1, 0)$ \\
			$\for_l$ & \L ukasiewicz t-conorm & $\for_l(x, y)=min(x + y, 1)$ \\
			\hline
			$\fand_m$ & G\"odel t-norm & $\fand_m(x, y)=min(x, y)$ \\
			$\for_m$ & G\"odel t-conorm & $\for_m(x, y)=max(x, y)$ \\
			\hline
			$\fand_p$ & product t-norm & $\fand_p(x, y)=x \cdot y$ \\
			$\for_p$ & product t-conorm & $\for_p(x, y)=x+ y -x\cdot y$ \\
			\hline
			$\neg_s$ & standard negator & $\neg_s(x) =
                        1-x$\\ \hline

			$\rar_r$ & R-implicator induced by $\fand_m$ &
                        $\rar_r\!\!(x, y)=\begin{cases}1 & \text{if}\ x \leq y\\y
                          & \text{otherwise}\end{cases}$\\
			$\rar_s$ & S-implicator induced by $\fand_m$
                        &  $\rar_s\!\!(x, y) =
                        max(1-x, y)$\\  
                        $\rar_l$ & Implicator induced by $\fand_l$ & $\rar_l\!\!(x,y) = min(1-x+y, 1)$ \\
\hline
\end{tabular}
\caption{Some t-norms, t-conorms, negator, and implicators}
\label{fig:operators}
\end{center}
\end{figure}

The \emph{truth value} of a fuzzy propositional formula $F$ under $I$, denoted $\u_I(F)$, is defined recursively as follows:  
\begin{itemize}
\item  for any atom $p\in\sigma$,\ \ $\u_I(p) = I(p)$;
\item  for any numeric constant ${c}$,\ \  $\u_I({c}) = c$;
\item  $\u_I(\neg F) = \neg(\u_I(F))$;
\item  $\u_I(F\odot G) = \odot(\u_I(F), \u_I(G))$ \ \ \ ($\odot\in\{\fand, \for, \frar\}$).
\end{itemize}
(For simplicity, we identify the symbols for the fuzzy operators with the truth value functions represented by them.)

\begin{definition}\label{def:fuzzy-m}\optional{def:fuzzy-m}
We say that a fuzzy interpretation $I$ {\em satisfies} a fuzzy formula~$F$ if $\u_I(F) = 1$, and denote it by $I\models F$. We call such $I$ a {\em fuzzy model} of~$F$.

We say that two formulas $F$ and $G$ are equivalent if $\u_I(F) = \u_I(G)$ for all interpretations~$I$, and denote it by $F\Lrar G$.
\end{definition}

\subsection{Review: Stable Models of Classical Propositional Formulas}\label{ssec:review-sm}  \optional{ssec:review-sm}

A {\em propositional signature} is a set of symbols called {\em atoms}. A propositional formula is defined recursively using atoms and the following set of primitive propositional connectives: $\bot,\, \top,\, \neg,\, \land,\ \lor,\ \rar$.

An {\em interpretation} of a propositional signature is a function from the signature into $\{\false,\true\}$. We identify an interpretation with the set of atoms that are true in it. 

A {\em model} of a propositional formula is an interpretation that {\em satisfies} the formula. 
According to \cite{fer05}, the models are divided into stable models and non-stable models as follows.
The {\em reduct} $F^X$ of a propositional formula $F$ relative to a set~$X$ of atoms is the formula obtained from $F$ by replacing every maximal subformula that is not satisfied by~$X$ with $\bot$. 
Set~$X$ is called a {\em stable model}  of $F$ if $X$ is a minimal set of atoms satisfying~$F^X$. 

Alternatively, the reduct $F^X$ can be defined recursively as follows:

\begin{definition}\label{def:classical-reduct}
\begin{itemize}
\item When $F$ is an atom or $\bot$ or $\top$, 
      $
      F^X = 
      \begin{cases}
          F & \text{if $X\models F$;} \\
          \bot & \text{otherwise;}
      \end{cases}
      $
      
\item $
      (\neg F)^X = 
      \begin{cases}
          \bot & \text{if $X\models F$;} \\
          \top & \text{otherwise;}
      \end{cases}
      $

\item For $\odot\in\{\land, \lor, \rar\}$,

      $
      (F\odot G)^X = 
      \begin{cases}
          F^X\odot G^X & \text{if $X\models F\odot G$;} \\
          \bot & \text{otherwise.}
      \end{cases}
      $
\end{itemize}
\end{definition}

In the next section, we extend this definition to cover fuzzy propositional formulas. 

\section{Definition and Examples} \label{sec:definition} 
   \optional{sec:definition}

\subsection{Reduct-based Definition} 

Let $\sigma$ be a fuzzy propositional signature, $F$ a fuzzy propositional formula of $\sigma$, and $I$ an interpretation of $\sigma$. 

\begin{definition}\label{def:fuzzy-reduct}
The {\em (fuzzy) reduct} of $F$ relative to $I$, denoted $F^I$, is defined recursively as follows:
\begin{itemize}
\item  For any fuzzy atom or numeric constant $F$, $F^I = F$; 
\item  $(\neg F)^I = \u_I(\neg F)$; 
\item  $(F\odot G)^I = (F^I\odot G^I)\fand_m {\u_I(F\odot G)}$, where 
       $\odot\in\{\fand, \for, \frar\}$.
\end{itemize}
\end{definition}

Compare this definition and Definition~\ref{def:classical-reduct}. They are structurally similar, but also have subtle differences. One of them is that in the case of binary connectives $\odot$, instead of distinguishing the cases between when $X$ satisfies the formula or not as in Definition~\ref{def:classical-reduct},  Definition~\ref{def:fuzzy-reduct} keeps a conjunction of $F^I\odot G^I$ and $\u_I(F\odot G)$.\footnote{%
In fact, a straightforward modification of the second subcases in Definition~\ref{def:classical-reduct} by replacing $\bot$ with some fixed truth values does not work for the fuzzy stable model semantics.}

Another difference is that in the case of an atom, Definition~\ref{def:fuzzy-reduct} is a bit simpler as it does not distinguish between the two cases. It turns out that the same clause can be applied to Definition~\ref{def:classical-reduct} (i.e., $F^X = F$ when $F$ is an atom regardless of $X\models F$), still yielding an equivalent definition of a Boolean stable model. So this difference is not essential.

For any two fuzzy interpretations $J$ and $I$ of signature~$\sigma$ and any subset~${\bf p}$ of~$\sigma$, we write $J\le^{\bf p} I$ if
\bi
\item  $J(p) = I(p)$ for each fuzzy atom not in ${\bf p}$, and 
\item  $J(p) \le I(p)$ for each fuzzy atom $p$ in ${\bf p}$.
\ei
We write $J<^{\bf p} I$ if $J\le ^{\bf p} I$ and $J\ne I$. We may simply write it as $J<I$ when ${\bf p}=\sigma$.


\begin{definition}\label{def:fuzzy-sm}\optional{def:fuzzy-sm}
We say that an interpretation $I$ is a {\em (fuzzy) stable model} of $F$ relative to~${\bf p}$ (denoted $I\models\sm[F; {\bf p}]$) if 
\bi
\item  $I\models F$, and 

\item  there is no interpretation $J$ such that $J<^{\bf p} I$ and $J$ satisfies $F^I$.
\ei
If ${\bf p}$ is the same as the underlying signature, we may simply write $\sm[F; {\bf p}]$ as $\sm[F]$ and drop the clause ``relative to ${\bf p}$.''

We call an interpretation $J$ such that $J<^{\bf p} I$ and $J$ satisfies $F^I$ as a {\em witness to dispute the stability} of $I$ (for $F$ relative to ${\bf p}$). In other words, a model of $F$ is stable if it has no witness to dispute its stability for $F$.
\end{definition}

Clearly, when ${\bf p}$ is empty, Definition~\ref{def:fuzzy-sm} reduces to the definition of a fuzzy model in Definition~\ref{def:fuzzy-m} simply because there is no interpretation $J$ such that $J<^\emptyset I$.

The definition of a reduct can be simplified in the cases of $\fand$ and $\for$, which are increasing in both arguments: 
\begin{itemize}
\item  $(F\fand G)^I = (F^I\fand G^I)$;\ \ \  $(F\for G)^I = (F^I\for G^I)$.
\end{itemize}

This is due to the following proposition, which can be proved by induction.
\begin{prop}\label{prop:monotone}\optional{prop:monotone}
For any interpretations $I$ and $J$ such that $J\le^{\bf p} I$, it holds that 
\[
\u_J(F^I)\le \u_I(F).
\]
\end{prop}

In the following, we will assume this simplified form of a reduct. 

Also, we may view $\neg F$ as shorthand for some fuzzy implication $F\rar 0$. For instance, what we called the standard negator can be derived from the residual implicator $\rar_l$ induced by \L ukasiewicz t-norm, defined as $\rar_l\!\!(x,y) = min(1-x+y, 1)$. In view of Proposition~\ref{prop:monotone}, for $J\le^{\bf p} I$,
\begin{align*}
 \u_J( (F\rar_l\,  0)^I ) & = \u_J( (F^I\rar_l\,  0)\fand_m \u_I(F\rar_l\,  0) )  \\
               & = \u_J( \u_I(F\rar_l\,  0) )  = \u_I(F\rar_l\,  0) = \u_I(\neg_s F).
\end{align*}
Thus the second clause of Definition~\ref{def:fuzzy-reduct} can be viewed as a special case of the third clause. 

\begin{example}\label{ex:pq} \optional{ex:pq}
Consider the fuzzy formula $F_1 = \neg_s q \rar_r\, p$ and the interpretation 
$I = \{(p, x), (q, 1-x)\}$, where $x\in [0,1]$.  $I\models F_1$ because 
\[
  \u_I(\neg_s q \rar_r\, p) =\  \rar_r\!\!(x, x)=1.
\]
The reduct $F_1^I$ is 
\[
   ((\neg_s q)^I\rar_r\, p)\fand_m 1
   \ \ \Lrar\ \ (\u_I(\neg_s q)\rar_r\, p) 
   \ \ \Lrar\ \ x\rar_r\, p
\]
$I$ can be a stable model only if $x=1$. Otherwise, $\{(p,x), (q,0)\}$ is a witness to dispute the stability of $I$.

On the other hand, for $F_2 = (\neg_s q\rar_r\, p)\fand_m (\neg_s p\rar_r\, q)$, for any value of $x\in[0,1]$, $I$ is a stable model. \footnote{%
Strictly speaking, $(1-x)$ in the reduct should be understood as the value of the arithmetic function applied to the arguments.}
\[
\ba {rcll}   
  F_2^I & = &  (\neg_s q\rar_r\, p)^I  & \fand_m\ (\neg_s p\rar_r\, q)^I  \\
  & \Lrar & ({\u_I(\neg_s q)}\rar_r\, p)\fand_m 1 & \fand_m\ ({\u_I(\neg_s p)}\rar_r\, q)\fand_m 1 \\
  & \Lrar & ({x}\rar_r\, p) & \fand_m\ ((1-x)\rar_r\, q) . 
\ea
\]
No interpretation $J$ such that $J<^{\{p,q\}} I$ satisfies $F_2^I$.
\end{example}

\begin{example}\label{ex:double-neg}\optional{ex:double-neg}
\bi
\item $F_1 = p\rar_r\, p$ is a tautology (i.e., every interpretation is a model of the formula), but not all models are stable. First, $I_1=\{(p,0)\}$ is a stable model. The reduct $F_1^{I_1}$ is 
\[
   (p^I\rar_r\, p^I)\fand_m 1 \ \ \Lrar\ \   p\rar_r\, p  \ \ \Lrar\ \   1
\]
and obviously, there is no witness to dispute the stability of $I$. 

No other interpretation $I_2 = \{(p, x)\}$ where $x>0$ is a stable model. The reduct $F_1^{I_2}$ is again equivalent to $1$, but $I_1$ serves as a witness to dispute the stability of $I_2$. 

\item 
$F_2 = \neg_s\neg_s p\rar_r\, p$ is equivalent to $F_1$, but their stable models are different. Any $I = \{(p,x)\}$, where $x\in [0,1]$, is a stable model of $F_2$. The reduct $F_2^I$ is 
\[
\ba {rcl}   
   ((\neg_s\neg_s p)^{I} \rar_r\, p^{I})\fand_m {1} 
      & \Lrar & {\u_I(\neg_s\neg_s p)} \rar_r\, p \\
      & \Lrar & {x}\rar_r\, p .
\ea
\]
No interpretation $J$ such that $J<^{\{p\}} I$ satisfies $F_2^I$.

\item 
Let $F_3$ be $\neg_s p\for_l p$. Any $I=\{(p,x)\}$, where $x\in [0,1]$, is a stable model of $F_3$. The reduct $F_3^I$ is  
\[
\ba {rcl}
   (\neg_s p)^I\for_l p^I \ \ \Lrar\ \ 
    \u_I(\neg_s p)\for_l p \ \ \Lrar\ \ (1-x)\for_l p .
\ea
\]
No interpretation $J$ such that $J <^p I$ satisfies the reduct. 
\ei
\end{example}

The following proposition extends a well-known fact about the relationship between a formula and its reduct in terms of satisfaction.
\begin{prop}\label{prop:reduct-satisfaction}
A (fuzzy) interpretation $I$ satisfies a formula $F$ if and only if $I$ satisfies $F^I$.
\end{prop}

For any fuzzy formula $F$, any interpretation $I$ and any set ${\bf p}$ of atoms, we say that $I$ is a {\em minimal} model of $F$ relative to ${\bf p}$ if $I$ satisfies $F$ and there is no interpretation $J$ such that $J<^{\bf p} I$ and $J$ satisfies $F$. Using this notion, Proposition~\ref{prop:reduct-satisfaction} tells us that Definition~\ref{def:fuzzy-sm} can be reformulated as follows. 
\begin{cor}
An interpretation $I$ is a {\em (fuzzy) stable model} of $F$ relative to~${\bf p}$ iff $I$ is a minimal model of $F^I$ relative to~${\bf p}$.
\end{cor}
 
This reformulation relies on the fact that $\fand_m$ is intended in the third bullet of Definition \ref{def:fuzzy-reduct} instead of an arbitrary fuzzy conjunction because we want the truth value of the ``conjunction'' of $F^I\rar G^I$ and $F\rar G$ to be either the truth value of $F^I\rar G^I$ or the truth value $F\rar G$.
Conjunctions that do not have this property lead to unintuitive behaviors, such as violating Proposition \ref{prop:reduct-satisfaction}. As an example, consider the formula
\[
  F = 0.6\rar_r\, (1 \rar_r\,  p)
\]
and the interpretation
\[
I = \left\{(p, 0.6)\right\}.
\]
Clearly, $I$ satisfies $F$. According to Definition~\ref{def:fuzzy-reduct}, 
\begin{align*}
 F^I &=(0.6\rar_r\,  (1\rar_r\,  p)^I)\fand_m \u_I(0.6 \rar_r\,  (1 \rar_r\,  p))\\
     &= (0.6\rar_r\,  ((1\rar_r\,  p)\fand_m \u_I(1\rar_r\,  p)))\fand_m 1\\
     &= 0.6\rar_r\,  ((1\rar_r\,  p)\fand_m 0.6)
\end{align*}
and
\begin{align*}
\u_I(F^I) &= (0.6\rar_r\,  ((1\rar_r\,  0.6)\fand_m 0.6))\\
        &= 0.6\rar_r\,  0.6 \ \ = \ \ 1
\end{align*}
So $I$ satisfies $F^I$ as well. However, if we replace $\fand_m$ by $\fand_l$ in the third bullet of Definition~\ref{def:fuzzy-reduct}, we get
\begin{align*}
\u_I(F^I) &= 0.6\rar_r\,  ((1\rar_r\, 0.6) \fand_l 0.6)\\
         &= 0.6\rar_r\,  0.2\ \  =\ \  0.2 \ \ < \ \ 1, 
\end{align*}
which indicates that $I$ does not satisfy $F^I$. Therefore, $I$ is a fuzzy stable model of~$F$ according to Definition~\ref{def:fuzzy-sm}, but it is not even a model of the reduct $F^I$ if $\fand_l$ were used in place of $\fand_m$ in the definition of a reduct.

Similarly, it can be checked that the same issue remains if we use $\fand_p$ in place of~$\fand_m$.

The following example illustrates how the commonsense law of inertia involving fuzzy truth values can be represented. 

\begin{example}\label{ex:default}
Let $\sigma$ be $\{p_0, np_0, p_1, np_1\}$ 
and let $F$ be $F_1\fand_m F_2$, where $F_1$ represents that
$p_1$ and $np_1$ are complementary, i.e., the sum of their truth
values is~$1$: \footnote{This is similar to the formulas used under the Boolean stable model semantics to express that two Boolean atoms $p_1$ and $np_1$ take complimentary values, i.e.,
\[
\neg(p_1 \wedge np_1)\wedge\neg\neg(p_1 \vee np_1).
\]}
\[ 
   F_1=\neg_{s}(p_1 \fand_l np_1) \fand_m \neg_{s} \neg_{s}(p_1 \for_l np_1).  
\]
$F_2$ represents that by default $p_1$ has the truth value of $p_0$, and
$np_1$ has the truth value of~$np_0$: \footnote{This is similar to the rules used in ASP to express the commonsense law of inertia, e.g.,
\[
p_0 \wedge \neg\neg p_1 \rightarrow p_1.
\]}
\[
\ba l
   F_2= ((p_0\fand_m \neg_{s} \neg_{s} p_1) \rar_r\, p_1) 
      \fand_m ((np_0\fand_m \neg_{s} \neg_{s}  np_1) \rar_r\, np_1). 
\ea
\]

One can check that the interpretation 
\[ 
  I_1 = \{(p_0, x), (np_0, 1-x), (p_1, x), (np_1, 1-x) \}
\]
($x$ is any value in $[0,1]$)
is a stable model of $F$ relative to $\{p_1, np_1\}$: $F^{I_1}$ is equivalent to 
\[
\ba l
   (p_0\fand_m x\rar_r\, p_1)\fand_m (np_0\fand_m (1-x)\rar_r\, np_1).
\ea
\]
No interpretation $J$ such that $J<^{\{p_1,np_1\}} I_1$ satisfies $F^{I_1}$. 

The interpretation 
\[ 
  I_2 = \{(p_0, x), (np_0, 1-x), (p_1, y), (np_1, 1-y) \}, 
\] 
where $y > x$, is not a stable model of $F$.  The reduct $F^{I_2}$ is equivalent to 
\[
\ba l
   (p_0\fand_m y\rar_r\, p_1)\fand_m (np_0\fand_m (1-y)\rar_r\, np_1),
\ea
\]
and the interpretation $\{(p_0, x), (np_0, 1-x), (p_1, x), (np_1, 1-y)\}$ serves as a witness to dispute the stability of $I_2$.

Similarly, when $y < x$, we can check that $I_2$ is not a stable model of $F$ relative to $\{p_1, np_1\}$. 

On the other hand, if we conjoin $F$ with $y\rar_r\, p_1$ and $(1-y) \rar_r\, np_1$ to yield $F\fand_m (y\rar_r\, p_1)\fand_m ((1-y)\rar_r\, np_1)$, then the default behavior is overridden, and $I_2$ is a stable model of $F\fand_m ({y} \rar_r\, p_1)\fand_m ((1-y)\rar_r\, np_1)$ relative to $\{p_1,np_1\}$.\footnote{One may wonder why the part $(1-y)\rar_r\, np_1$ is also needed. It can be checked that if we drop the part $(1-y)\rar_r\, np_1$ and have $y$ less than $x$, then $I_2$ (with $y<x$) is not a stable model of $F\fand_m ({y} \rar_r\, p_1)$ relative to $\{p_1,np_1\}$ because $J=\{(p_0, x), (np_0, 1-x), (p_1, y), (np_1, 1-x)\}$ disputes the stability of~$I_2$.}

This behavior is useful in expressing the commonsense law of inertia involving fuzzy values. 
Suppose $p_0$ represents the value of fluent $p$ at time $0$, and $p_1$ represents the value at time $1$. Then $F$ states that ``by default, the fluent retains the previous value.'' The default value is overridden if there is an action that sets $p$ to a different value. 

This way of representing the commonsense law of inertia is a straightforward extension of the solution in ASP. 
\end{example}

\begin{example}\label{ex:trust}
The trust example in the introduction can be formalized in the fuzzy stable model semantics as follows. 
Below $x$, $y$, $z$ are schematic variables ranging over people, and $t$ is
a schematic variable ranging over time steps. $\i{Trust}(x,y,t)$
is a fuzzy atom representing that ``$x$ trusts $y$ at time $t$.''
Similarly, $\i{Distrust}(x,y,t)$ is a fuzzy atom representing that 
``$x$ distrusts $y$ at time $t$.'' 

The trust relation is reflexive: 
\[ 
  F_1 = \i{Trust}(x,x,t) .
\]

The trust and distrust degrees are complementary, i.e., their sum is
$1$ (similar to $F_1$ in Example~\ref{ex:default}): 
\[ 
\ba l 
   F_2 = \neg_{s} (\i{Trust}(x,y,t)\fand_l \i{Distrust}(x,y,t)),  \\
   F_3 = \neg_{s} \neg_{s} (\i{Trust}(x,y,t)\for_l\i{Distrust}(x,y,t)).

\ea 
\]

Initially, if $x$ trusts $y$ to degree $d_1$ and $y$ trusts $z$ to
degree $d_2$, then we assume $x$ trusts $z$ to degree $d_1\times d_2$; furthermore 
the initial distrust degree is $1$ minus the initial trust degree.
\[ 
\ba l
   F_4 = \i{Trust}(x,y,0)\fand_p\i{Trust}(y,z,0)\rar_r\,\i{Trust}(x,z,0), \\
   F_5 = \neg_{s}\i{Trust}(x,y,0)\rar_r\,\i{Distrust}(x,y,0).
\ea 
\]

The inertia assumption (similar to $F_2$ in Example~\ref{ex:default}): 
\[
\ba l
    F_6 = \i{Trust}(x,y,t)\fand_m \neg_{s} \neg_{s}\i{Trust}(x,y,t\!+\!1)
          \rar_r\, \i{Trust}(x,y,t\!+\!1),  \\
    F_7 = \i{Distrust}(x,y,t)\fand_m \neg_{s} \neg_{s} \i{Distrust}(x,
    y, t\!+\!1) \rar_r\, \i{Distrust}(x,y,t\!+\!1).
\ea 
\]

A conflict increases the distrust
degree by the conflict degree: 
\[
\ba l
   F_8= \i{Conflict}(x,y,t)\for_l \i{Distrust}(x,y,t) 
        \rar_r\,\i{Distrust}(x, y, t\!+\!1), \\
\ea 
\]

Let $F_{TW}$ be $F_1 \otimes_m F_2 \otimes_m \dots \otimes_m F_8$. 
Suppose we have the formula $F_{Fact}=\i{Fact}_1\fand_m\i{Fact}_2$ that
gives the initial trust degree.
\[
\ba l
   \i{Fact}_1 = {0.8}\rar_r\,\i{Trust}(\i{Alice}, \i{Bob}, 0), \\
   \i{Fact}_2 = {0.7}\rar_r\,\i{Trust}(\i{Bob}, \i{Carol}, 0).
\ea 
\]
Although there is no fact about how much $\i{Alice}$ trusts
$\i{Carol}$, any stable model of $F_{TW}\fand_m F_{Fact}$ assigns
value $0.56$ to the atom $\i{Trust}(\i{Alice}, \i{Carol}, 0)$. On the other
hand, the stable model assigns value $0$ to
$\i{Trust}(\i{Alice},\i{David},0)$ due to the closed world assumption under
the stable model semantics.

When we conjoin  $F_{TW}\fand F_{Fact}$ with 
${0.2}\rar_r\,\i{Conflict}(\i{Alice},\i{Carol}, 0)$, 
the stable model of 
\[
F_{TW}\fand_m
F_{Fact}\fand_m ({0.2}\rar_r\,\i{Conflict}(\i{Alice},\i{Carol},
0))
\]
manifests that the trust degree between $\i{Alice}$ and $\i{Carol}$
decreases to $0.36$ at time~$1$. More generally, if we have more actions that
change the trust degree in various ways, by specifying the entire
history of actions, we can determine the evolution of the trust
distribution among all the participants. Useful decisions can
be made based on this information. For example, $\i{Alice}$ may decide not to
share her personal pictures to those whom she trusts less than degree
$0.48$.
\end{example}

Note that this example, like Example~\ref{ex:default}, uses nested connectives, such as $\neg_{s}\neg_{s}$, that are not available in the syntax of FASP  considered in earlier work, such as \cite{lukasiewicz06fuzzy,janssen12reducing}.

It is often assumed that, for any fuzzy rule arrow $\ar$, it holds that  $\ar\!\!(x, y) = 1$ iff $x\geq y$ \cite{damasio01monotonic}. This condition is required to use  $\ar$ to define an immediate consequence operator for a positive program whose least fixpoint coincides with the unique minimal model. (A positive program is a set of rules of the form $a\ar b_1\fand\dots\fand b_n$, where $a, b_1, \dots, b_n$ are atoms.) Notice that $\rar_r$ in Figure~\ref{fig:operators} satisfies this condition, but $\rar_s$ does not.

\section{Generalization and Alternative Definitions}\label{sec:gen-alt}\optional{sec:gen-alt}

\subsection{$y$-Stable Models} 

While the fuzzy stable model semantics presented in the previous section allows atoms to have many values, like ASP, it holds on to the two-valued concept of satisfaction, i.e., a formula is either satisfied or not.
In a more flexible setting we may allow a formula to be partially satisfied to a certain degree. 

First, we generalize the notion of satisfaction to allow partial degrees of satisfaction as in~\cite{nieuwenborgh07anintroduction}.

\begin{definition}
For any real number $y\in [0,1]$, we say that a fuzzy interpretation $I$ {\em $y$-satisfies} a fuzzy formula~$F$ if $\u_I(F)\ge y$, and denote it by $I\models_y F$. We call $I$ a fuzzy $y$-model of~$F$. 
\end{definition}

Using this generalized notion of satisfaction, it is straightforward to generalize the definition of a stable model to incorporate partial degrees of satisfaction.

\begin{definition}
We say that an interpretation $I$ is a fuzzy $y$-stable model of $F$ relative to~${\bf p}$ (denoted $I\models_y\sm[F; {\bf p}]$) if 
\bi
\item  $I\models_y F$ and
\item  there is no interpretation $J$ such that $J<^{\bf p} I$ and $J\models_y F^{I}$. 
\ei
\end{definition}

Fuzzy models and fuzzy stable models as defined in Section~\ref{ssec:review-fuzzy} and Section~\ref{sec:definition} are fuzzy $1$-models and fuzzy $1$-stable models according to these generalized definitions. 

\begin{example}\label{ex:not-p-implies-q}\optional{ex:not-p-implies-q}
Consider the fuzzy formula $F = \neg_{s}\, p \rar_r\, q$ and the interpretation $I = \{(p, 0), (q, 0.6)\}$. $I\models_{0.6} F$ because $\u_I(\neg_s\, p \rar_r\, q) = \rar_r\!\!(1, 0.6)=0.6$. 

The reduct $F^I$ is 
\[
   ((\neg_s p)^I\rar_r\, q)\fand_m 1\ \ \Lrar\ \ (\u_I(\neg_s p)\rar_r\, q) 
   \ \ \Lrar\ \  1\rar_r\, q
\]
Clearly, for any $J$ such that $J<^{\{p,q\}} I$, we observe that $J\not\models_{0.6} F^I$. Hence, $I$ is a $0.6$-stable model of $F$.
\end{example}

On the other hand, the generalized definition is not essential in the sense that it can be reduced to the special case as follows. 

\begin{thm}\label{thm:ystable-1stable}\optional{thm:ystable-1stable}
For any fuzzy formula $F$, an interpretation $I$ is a $y$-stable model of $F$ relative to ${\bf p}$ iff $I$ is a $1$-stable model of $y\rar F$ relative to ${\bf p}$ as long as the implication $\rar$ satisfies the condition $\rar\!\!(x, y) = 1$ iff $y\geq x$. 
\end{thm}

For example, in accordance with Theorem~\ref{thm:ystable-1stable},
$\{(p,0.6)\}$ is a $0.6$-stable model of $\neg_s\, p \rar_r\, q$, as well as a $1$-stable model of $0.6\rar_r\ (\neg_s\, p \rar_r\, q)$.

\subsection{Second-Order Logic Style Definiton} 

In~\cite{ferraris11stable}, second-order logic was used to define the stable models of a first-order formula, which is equivalent to the reduct-based definition when the domain is finite. 
Similarly, but instead of using second-order logic, we can express the same concept using auxiliary atoms that do not belong to the original signature.

Let $\sigma$ be a set of fuzzy atoms, and let ${\bf p}=(p_1,\dots, p_n)$ be a list of distinct atoms belonging to~$\sigma$, and let ${\bf q}=(q_1,\dots,q_n)$ be a list of new, distinct fuzzy atoms not belonging to $\sigma$. For two interpretations $I$ and $J$ of $\sigma$ that agree on all atoms in $\sigma\setminus {\bf p}$,   
$I\cup J^{\bf p}_{\bf q}$ denotes the interpretation of $\sigma\cup{\bf q}$ such that
\bi
\item  $(I\cup J^{\bf p}_{\bf q})(p) = I(p)$ for each atom $p$ in $\sigma$, and

\item  $(I\cup J^{\bf p}_{\bf q})(q_i)=J(p_i)$ for each $q_i\in {\bf q}$.
\ei

For any fuzzy formula~$F$ of signature~$\sigma$, $F^*({\bf q})$ is defined as follows. 
                                                  
\begin{itemize}
\item  $p_i^* = q_i$ for each $p_i\in{\bf p}$; 
\item  $F^* = F$ for any atom $F\not\in {\bf p}$ or any numeric constant $F$;
\item  $(\neg F)^* = \neg F$;
\item  $(F\fand G)^* = F^*\fand G^*$; \qquad
       $(F\for G)^* = F^* \for G^*$;
\item  $(F\frar G)^* = (F^*\frar G^*)\fand_m (F \rar G)$.
\end{itemize}

\begin{thm}\label{thm:fuzzy-sm}
A fuzzy interpretation $I$ is a fuzzy stable model of $F$ relative to ${\bf p}$ iff 
\bi 
\item  $I \models F$, and
\item  there is no fuzzy interpretation $J$ such that  $J<^{\bf p} I$ and $I\cup J^{\bf p}_{\bf q}\models F^*({\bf q})$.
\ei
\end{thm}

\section{Relation to Boolean-Valued Stable Models} \label{sec:rel-bool} \optional{sec:rel-bool}

The Boolean stable model semantics in Section~\ref{ssec:review-sm} can be embedded into the fuzzy stable model semantics as follows: 

For any classical propositional formula $F$, define $F^\mi{fuzzy}$ to be the fuzzy propositional formula obtained from $F$ by replacing $\bot$ with ${0}$,\  $\top$ with ${1}$,\  $\neg$ with $\neg_{s}$,\ $\land$ with $\fand_m$,\ $\lor$ with $\for_m$, and $\rar$ with $\rar_s$.   We identify the signature of $F^\mi{fuzzy}$ with the signature of $F$. Also, for any propositional interpretation $I$, we define the corresponding fuzzy interpretation $I^\mi{fuzzy}$ as 
\bi
\item  ${I^\mi{fuzzy}}(p) = 1$ if $I(p)=\true$; 
\item  ${I^\mi{fuzzy}}(p) = 0$ otherwise.
\ei

The following theorem tells us that the Boolean-valued stable model
semantics can be viewed as a special case of the fuzzy stable model semantics. 

\begin{thm}\label{thm:cl-fuzzy-sm}\optional{thm:cl-fuzzy-sm} 
For any classical propositional formula $F$ and any classical propositional interpretation $I$, $I$ is a stable model of $F$ relative to ${\bf p}$ iff $I^\mi{fuzzy}$ is a stable model of $F^\mi{fuzzy}$ relative to ${\bf p}$.
\end{thm}

\begin{example}
Let $F$ be the classical propositional formula $\neg q\rar p$. $F$ has
only one stable model $I = \{p\}$.
Likewise, as shown in Example~\ref{ex:pq}, $F^\mi{fuzzy}=\neg_{s} q\rar_s p$
has only one fuzzy stable model $I^\mi{fuzzy}=\{(p, 1), (q, 0)\}$.
\end{example}

Theorem~\ref{thm:cl-fuzzy-sm} may not necessarily hold for different selections of fuzzy operators, as illustrated by the following example.

\begin{example}
Let $F$ be the classical propositional formula $p\lor p$. Classical interpretation $I = \{p\}$ is a stable model of $F$. However, $I^\mi{fuzzy}= \{(p,1)\}$ is not a stable model of $F^\mi{fuzzy} = p\for_l p$ because there is $J = \{(p, 0.5)\} < I$ that satisfies $(F^\mi{fuzzy})^I = p\for_l p$.
\end{example}

However, one direction of Theorem~\ref{thm:cl-fuzzy-sm} still holds for different selections of fuzzy operators. 

\begin{thm} \label{thm:fuzzy-cl-sm}\optional{thm:fuzzy-cl-sm}
For any classical propositional formula $F$, let $F_1^\mi{fuzzy}$ be the fuzzy formula obtained from $F$ by replacing $\bot$ with ${0}$,  $\top$ with ${1}$, $\neg$ with any fuzzy negation symbol, $\land$ with any fuzzy conjunction symbol, $\lor$ with any fuzzy disjunction symbol, and $\rar$ with any fuzzy implication symbol. For any classical propositional interpretation $I$, if $I^\mi{fuzzy}$ is a fuzzy stable model of $F_1^\mi{fuzzy}$ relative to ${\bf p}$, then $I$ is a Boolean stable model of $F$ relative to ${\bf p}$.
\end{thm}

\section{Relation to Other Approaches to  Fuzzy  ASP} \label{sec:rel-fuzzyasp}  \optional{sec:rel-fuzzyasp}

\subsection{Relation to Stable Models of Normal FASP Programs} \label{ssec:normal-fasp}\optional{ssec:normal-fasp}

A normal FASP program \cite{lukasiewicz06fuzzy} is a finite set of rules of the form
\[ 
  a\ \ar\ b_1 \fand \dots \fand b_m \fand \neg b_{m+1}\fand \dots \fand \neg b_n, 
\] 
where $n\geq m \geq 0$, $a, b_1, \dots, b_n$ are fuzzy atoms or
numeric constants in $[0,1]$, and $\fand$ is any fuzzy conjunction. We identify the rule with the fuzzy implication 
\[ 
   b_1\fand\dots\fand b_m\fand\neg_{s} b_{m+1}\fand\dots\fand
   \neg_{s} b_n\ \rar_r\ a, 
\] 
which allows us to say that a fuzzy interpretation $I$ of signature~$\sigma$ {\em satisfies} a rule $R$ if $\u_I(R)=1$. $I$ {\em satisfies} an FASP program $\Pi$ if $I$ satisfies every rule in $\Pi$. 

According to \cite{lukasiewicz06fuzzy}, an interpretation $I$ is a {\em fuzzy answer set} of a normal FASP program $\Pi$ if $I$ satisfies $\Pi$, and no interpretation $J$ such that $J<^\sigma I$ satisfies the reduct of $\Pi$ w.r.t.~$I$, which is the program obtained from $\Pi$ by replacing each negative literal $\neg b$ with the constant for~$1-I(b)$. 

\begin{thm}\label{thm:normal-fuzzysm}\optional{thm:normal-fuzzysm}
For any normal FASP program $\Pi=\{r_1, \dots, r_n\}$, let $F$ be the fuzzy formula $r_1\fand\dots\fand r_n$, where $\fand$ is any fuzzy conjunction. An interpretation $I$ is a fuzzy answer set of $\Pi$ in the sense of \cite{lukasiewicz06fuzzy} if and only if $I$ is a stable model of $F$. 
\end{thm}

\begin{example}
Let $\Pi$ be the following program
\[  
\ba l
   p \ar \neg q \\
   q \ar \neg p.  
\ea
\]
The answer sets of $\Pi$ according to~\cite{lukasiewicz06fuzzy} are $\{(p, x), (q, 1-x)\}$, where $x$ is any value in $[0, 1]$: 
the corresponding fuzzy formula $F$ is 
$(\neg_{s} q \rar_r\, p) \fand_m (\neg_{s} p \rar_r\, q)$;\ \ 
As we observed in Example~\ref{ex:pq}, its stable models are $\{(p, x), (q, 1-x)\}$, where $x$ is any real number in $[0, 1]$. 
\end{example}

\subsection{Relation to Fuzzy Equilibrium Logic} \label{ssec:fuzzy-equil} \optional{ssec:fuzzy-equil}

Like the fuzzy stable model semantics introduced in this paper, fuzzy equilibrium logic~\cite{schockaert12fuzzy} generalizes fuzzy ASP programs to arbitrary propositional formulas, but its definition is quite complex as it is based on some complex operations on pairs of intervals and considers strong negation as one of the primitive connectives. Nonetheless, we show that fuzzy equilibrium logic is essentially equivalent to the fuzzy stable model semantics where all atoms are subject to minimization.

\subsubsection{Review: Fuzzy Equilibrium Logic}

We first review the definition of fuzzy equilibrium logic from~\cite{schockaert12fuzzy}. The syntax is the same as the one we reviewed in Section~\ref{ssec:review-fuzzy} except that a new connective $\sneg\ $ (strong negation) may appear in front of atoms.\footnote{The definition from~\cite{schockaert12fuzzy} allows strong negation in front of any formulas. We restrict its occurrence only in front of atoms as usual in answer set programs.} 
For any fuzzy propositional signature~$\sigma$, a (fuzzy N5) \emph{valuation} is a mapping from $\{h, t\}\times\sigma$ to subintervals of $[0, 1]$ such that  $V(t,a)\subseteq V(h,a)$ for each atom $a\in\sigma$.
For $V(w, a)=[u, v]$, where $w\in\{h, t\}$, we write $V^-(w,a)$ to denote the lower bound~$u$ and $V^+(w,a)$ to denote the upper bound~$v$. The {\em truth value} of a fuzzy formula under $V$ is defined as follows.
\begin{itemize}
\item  $V(w, {c}) = [c,\ c]$ for any numeric constant ${c}$;

\item  $V(w, \sneg a) = [1-V^+(w, a), 1-V^-(w, a)]$, where
       $\sim$ is the symbol for strong negation; 

\item  $V(w, F\fand G) = 
       [V^-(w,F)\fand V^-(w,G),\ \ V^+(w,F)\fand V^+(w,G)]$; \footnote{%
For readability, we write the infix notation $(x\odot y)$ in place of $\odot (x,y)$.}

\item  $V(w, F\for G) = 
       [V^-(w,F)\for V^-(w,G),\ \ V^+(w,F)\for V^+(w,G)]$;

\item  $V(h,\neg F) = [1-V^-(t, F),\ \ 1-V^-(h, F)]$;

\item  $V(t,\neg F) = [1-V^-(t, F),\ \ 1-V^-(t, F)]$;

\item  $V(h, F\rar G) = 
          [min(V^-(h,F)\rar V^-(h,G), V^-(t,F)\rar V^-(t,G)),  \\
~~\hspace{9cm} V^-(h,F)\rar V^+(h, G)]$;

\item $V(t, F\rar G) = [V^-(t,F)\rar V^-(t,G),\ \
           V^-(t,F)\rar V^+(t,G)]$.
\end{itemize}

A valuation $V$ is a (fuzzy N5) model of a formula $F$ if $V^-(h,F)=1$, which implies $V^+(h, F)=V^-(t, F)=V^+(t, F)=1$. For two valuations $V$ and $V'$, we say $V'\preceq V$ if $V'(t, a)=V(t, a)$ and $V(h, a)\subseteq V'(h, a)$ for all atoms $a$.
We say $V'\prec V$ if $V'\preceq V$ and $V'\ne V$.\ \ 
We say that a model $V$ of $F$ is {\em h-minimal} if there is no model $V'$ of $F$ such that $V'\prec V$.
An h-minimal fuzzy N5 model $V$ of $F$ is a \emph{fuzzy equilibrium model} of $F$ if $V(h,a)=V(t,a)$ for all atoms $a$.

\subsubsection{In the Absence of Strong Negation}

We first establish the correspondence between fuzzy stable models and fuzzy equilibrium models in the absence of strong negation. 
As in \cite{schockaert12fuzzy}, we assume that the fuzzy negation
$\neg$ is $\neg_{s}$. 

Notice that a fuzzy equilibrium model assigns an interval of values to each atom, rather than a single value as in fuzzy stable models. This accounts for the complexity of the definition of a fuzzy model. However, it turns out that in the absence of strong negation, all upper bounds assigned by a fuzzy equilibrium model are $1$.

\begin{lemma}\label{thm_ub_1}
Given a formula $F$ containing no strong negation, any equilibrium model $V$ of $F$ satisfies $V^+(h, a)=V^+(t, a)=1$ for all atoms $a$.
\end{lemma}

Therefore, in the absence of strong negation, any equilibrium model can be identified with a fuzzy interpretation as follows. For any valuation $V$, we define a fuzzy interpretation $\I_V$ as 
$\I_V(p) = V^-(h,p)$ for each atom $p\in\sigma$.

The following theorem asserts that there is a 1-1 correspondence between fuzzy equilibrium models and fuzzy stable models. 

\begin{thm}\label{thm:equil-sm-nostrneg}\optional{thm:equil-sm-nostrneg}
Let $F$ be a fuzzy propositional formula of $\sigma$ that contains no
strong negation.
\begin{itemize}
\item[(a)]  A valuation $V$ of~$\sigma$ is a fuzzy equilibrium model
  of $F$ iff $V^-(h,p)=V^-(t,p)$, $V^+(h,p)=V^+(t,p)=1$ for all
  atoms $p$ in $\sigma$ and $\I_V$ is a stable model of $F$ relative to
  $\sigma$. 

\item[(b)] An interpretation $I$ of~$\sigma$ is a stable model of
  $F$ relative to $\sigma$ iff $I = \I_V$ for some fuzzy equilibrium
  model $V$ of $F$. 
\end{itemize}
\end{thm}

\subsubsection{Allowing Strong Negation} \label{sssec:strneg}

In this section, we extend the relationship between fuzzy equilibrium
logic and our stable model semantics by allowing strong negation. 
This is done by simulating strong negation by new atoms in our
semantics. 

Let $\sigma$ denote the signature. For a fuzzy formula $F$ over $\sigma$ that may contain strong negation, define $F^\prime$ over $\sigma \cup \{np \mid p \in \sigma\}$ as the formula obtained from $F$ by replacing every strongly negated atom $\sneg p$ with a new atom $np$. The transformation $nneg(F)$ (``{\sl n}o strong {\sl neg}ation'') is defined as $nneg(F)=F^\prime \otimes_m \underset{p\in \sigma}{\bigotimes_m}\neg_s(p \otimes_l np)$. 

For any valuation $V$ of signature~$\sigma$, we define the valuation $nneg(V)$ of~$\sigma \cup \left\{np \mid p \in \sigma\right\}$ as 
\[ 
\begin{cases}
   nneg(V)(w, p) =[V^-(w, p), 1]  \\ 
   nneg(V)(w, np) = [1-V^+(w, p), 1]
\end{cases}
\]
for all atoms $p\in\sigma$. Clearly, for every valuation
$V$ of $\sigma$, there exists a corresponding interpretation
$\I_{nneg(V)}$ of $\sigma \cup \left\{np \mid p \in \sigma\right\}$. 
On the other hand, there exists an interpretation $I$ of $\sigma \cup \{np \mid p \in \sigma\}$ for which there is no corresponding valuation $V$ of $\sigma$ such that $I=\I_{nneg(V)}$.

\begin{example}
Suppose $\sigma=\{p\}$. For the valuation $V$ such that $V(w,p) = [0.2, 0.7]$, $nneg(V)$ is a valuation of $\{p, np\}$ such that 
\begin{center}
\text{$nneg(V)(w,p) = [0.2, 1]$  and $nneg(V)(w,np) = [0.3, 1]$.} 
\end{center}
Further, $\I_{nneg(V)}$ is an interpretation of $\{p, np\}$ such that 
\begin{center}
\text{$\I_{nneg(V)}(p) = 0.2$ and $\I_{nneg(V)}(np)=0.3$.}
\end{center}
On the other hand, the interpretation $I=\left\{(p, 0.6), (np, 0.8)\right\}$ of~$\sigma\cup \left\{np \mid p \in \sigma\right\}$ has no corresponding valuation $V$ of $\sigma$ such that $I=\I_{nneg(V)}$ because $[0.6, 0.2]$ is not a valid valuation.
\end{example}

The following proposition asserts that strong negation can be eliminated in favor of new atoms, extending the well-known results with the Boolean stable model semantics \cite[Section~8]{ferraris11stable} to fuzzy formulas.

\begin{prop}\label{lem:eqmodel_strong_neg}
For any fuzzy formula $F$ that may contain strong negation, a valuation $V$ is an equilibrium model of $F$ iff $nneg(V)$ is an equilibrium model of $nneg(F)$.
\end{prop}

Proposition~\ref{lem:eqmodel_strong_neg} allows us to extend the 1-1 correspondence between fuzzy equilibrium models and fuzzy stable models in Theorem~\ref{thm:equil-sm-nostrneg} to any formula that contains strong negation.

\begin{thm}\label{thm:equil-sm}\optional{thm:equil-sm}
For any fuzzy formula $F$ of signature $\sigma$ that may contain strong negation, 
\begin{itemize}
\item[(a)]  A valuation $V$ of $\sigma$ is a fuzzy equilibrium model
  of $F$ iff $V(h,p)=V(t,p)$ for all atoms $p$ in~$\sigma$ and $\I_{nneg(V)}$ is a
  stable model of $\i{nneg}(F)$ relative to $\sigma\cup\{np \mid
  p\in\sigma\}$.

\item[(b)] An interpretation $I$ of $\sigma\cup\{np \mid p\in\sigma\}$ is
  a stable model of $\i{nneg}(F)$ relative to $\sigma\cup\{np \mid
  p\in\sigma\}$ iff $I=\I_{nneg(V)}$ for some fuzzy equilibrium model $V$ of
  $F$. 
\end{itemize}
\end{thm}

\begin{example}
For fuzzy formula 
$F=({0.2} \rar_r\, p)\fand_m ({0.3} \rar_r\, \sneg p)$, 
formula $nneg(F)$ is
\[ ({0.2} \rar_r\, p)\fand_m
   ({0.3} \rar_r\, np) \fand_m
  \neg_s(p \fand_l np).
\]
One can check that the valuation $V$ defined as $V(w, p)=[0.2, 0.7]$ is the only equilibrium model of $F$, and the interpretation $\I_{nneg(V)}=\{(p, 0.2), (np, 0.3)\}$ is the only fuzzy stable model of $nneg(F)$.
\end{example}

This idea of eliminating strong negation in favor of new atoms was used in Examples~\ref{ex:default} and \ref{ex:trust}.

The correspondence between fuzzy equilibrium models and fuzzy stable models indicates that the complexity analysis for fuzzy equilibrium logic applies to fuzzy stable models as well. In~\cite{schockaert12fuzzy} it is shown that deciding whether a formula has a fuzzy equilibrium model is $\Sigma_2^P$-hard, which applies to the fuzzy stable model semantics as well. The same problem for a normal FASP programs, which can be identified with a special case of the fuzzy stable model semantics as shown in Section~\ref{ssec:normal-fasp}, is NP-hard.  Complexity analyses for other special cases based on restrictions on fuzzy operators, or the presence of cycles in a program have been studied in~\cite{blondeel14complexity}.

\section{Some Properties of Fuzzy Stable Models} \label{sec:properties} 
   \optional{sec:properties}

In this section, we show that several well-known properties of the Boolean stable model semantics can be naturally extended to the fuzzy stable model
semantics.

\subsection{Theorem on Constraints}

In answer set programming, constraints---rules with $\bot$ in the head---play an important role in view of the fact that adding a constraint eliminates the stable models that ``violate'' the constraint. 
The following theorem is the counterpart of Theorem~3 from~\cite{ferraris11stable} for fuzzy propositional formulas. 

\begin{thm}\label{thm:constraint}\optional{thm:constraint}
For any fuzzy formulas $F$ and $G$, $I$ is a stable model of \hbox{$F\fand\neg G$} (relative to ${\bf p}$) if and only if $I$ is a stable model of $F$ (relative to ${\bf p}$) and $I\models\neg G$.
\end{thm}

\begin{example}
Consider $F=(\neg_{s} p \rar_r\, q)\fand_m(\neg_{s} q \rar_r\, p)\fand_m \neg_{s} p$. Formula~$F$ has only one stable model $I = \{(p, 0), (q, 1)\}$, which is the only stable model of \\
\hbox{$(\neg_{s} p \rar_r\, q)\fand_m(\neg_{s} q \rar_r\, p)$} that satisfies $\neg_{s} p$.
\end{example}

\BOCC
If we consider a more general $y$-stable model, then only one direction holds.

\begin{thm}\label{thm:constraint2} \optional{thm:constraint2}
For any fuzzy formulas $F$ and $G$, if $I$ is a $y$-stable model of $F\fand\neg G$ (relative to ${\bf p}$), then $I$ is a $y$-stable model of $F$ (relative to ${\bf p}$) and $I\models_y\neg G$.
\end{thm}

\begin{example}
The other direction, that is, ``if $I$ is a $y$-stable model of $F$ and 
$I \models_y\neg G$, then $I$ is a $y$-stable model of $F\fand \neg
G$,'' does not hold in general. For example, consider $F=G=p$ and
$\fand$ to be $\fand_l$, and 
interpretation $I=\{(p,0.4)\}$. 
Clearly, $I$ is a $0.4$-stable model of $p$ and $I\models_{0.4} \neg
p$, but $I$ is not a $0.4$-stable model of $p\fand_l \neg p$. In fact, $I$ is not even a $0.4$-model of the formula.
\end{example}
\EOCC

\subsection{Theorem on Choice Formulas} 

In the Boolean stable model semantics, formulas of the form $p\lor\neg p$ are called {\em choice formulas}, and adding them to the program makes atoms $p$ exempt from minimization. Choice formulas have been shown to be useful in constructing an ASP program in the ``Generate-and-Test'' style. This section shows their counterpart in the fuzzy stable model semantics. 

For any finite set of fuzzy atoms ${\bf p}=\{p_1,\dots, p_n\}$, the expression ${\bf p}^{\rm ch}$ stands for the choice formula 
\[ 
   (p_1\for_l\neg_{s} p_1)\fand\dots\fand(p_n\for_l\neg_s p_n), 
\] 
where $\fand$ is any fuzzy conjunction.

The following proposition tells us that choice formulas are tautological. 

\begin{prop}\label{lem-choice-tautology}\optional{lem-choice-tautology}
For any fuzzy interpretation $I$ and any finite set ${\bf p}$ of fuzzy atoms, \hbox{$I\models {\bf p}^{\rm ch}$}.\footnote{This proposition may not hold if $\for_l$ in the choice formula is replaced by an arbitrary fuzzy disjunction. For example, consider using $\for_m$ instead. Clearly, the interpretation $I = \left\{(p, 0.5)\right\} \not\models p\for_m \neg_s p$.}
\end{prop}

Theorem~\ref{thm:choice} is an extension of Theorem~2
from~\cite{ferraris11stable}. 

\begin{thm} \label{thm:choice}
\bi
\item[(a)] For any real number $y\in [0,1]$, if $I$ is a $y$-stable model of $F$ relative to ${\bf p}\cup {\bf q}$, then $I$ is a $y$-stable model of $F$ relative to
  ${\bf p}$.
\item[(b)] $I$ is a $1$-stable model of $F$ relative to ${\bf p}$ iff 
   $I$ is a $1$-stable model of $F\fand {\bf q}^{\rm ch}$ relative
   to ${\bf p}\cup {\bf q}$.
\ei
\end{thm}

Theorem~\ref{thm:choice}~(b) does not hold for an arbitrary threshold $y$ (i.e., if ``$1-$'' is replaced with ``$y-$''). For example, consider $F=\neg_s\neg_s q$ and $I = \{(q, 0.5)\}$. Clearly, $I$ is a $0.5$-model of $F$, and thus $I$ is a $0.5$-stable model of $F$ relative to $\emptyset$. However, $I$ is not a $0.5$-stable model of $F\fand_m \{q\}^{\rm ch}=\neg_s \neg_s q\fand_m (q\for_l \neg_s q)$ relative to $\emptyset \cup \{q\}$, as witnessed by $J =\{(q, 0)\}$.

Since the $1$-stable models of $F$ relative to $\emptyset$ are the models of $F$, it follows from Theorem~\ref{thm:choice} (b)  that the {\em $1$-stable models} of $F\fand\sigma^{\rm ch}$ relative to the whole signature $\sigma$ are exactly the {\em $1$-models} of $F$.

\begin{cor}\label{cor:choice}
Let $F$ be a fuzzy formula of a finite signature $\sigma$.
$I$ is a model of $F$ iff $I$ is a stable model of~$F\fand \sigma^{\rm ch}$ relative to~$\sigma$.
\end{cor}

\begin{example}
Consider the fuzzy formula $F=\neg_{s} q\rar_r\, p$ in Example~\ref{ex:pq}, which has only one stable model $\{(p,1), (q,0)\}$, although any interpretation $I = \{(p, x), (q, 1-x)\}$ is a model of $F$. 
In accordance with Corollary~\ref{cor:choice}, we check that any $I$ is a stable model of $G = F\fand_m (p\for_l \neg_{s} p) \fand_m (q\for_l \neg_{s} q)$.
The reduct $G^I$ is equivalent to  
\[
  ((1-\u_I(q)) \rar_r\, p) \fand_m (p\for_l (1-\u_I(p))) 
  \fand_m (q\for_l (1-\u_I(q))) .
\]
It is clear that any interpretation $J$ that satisfies $G^I$ should be such that 
$\u_J(p) \ge \u_I(p)$ and $\u_J(q)\ge \u_I(q)$, so there is no witness to dispute the stability of $I$. 
\end{example}

\section{Other Related Work} \label{sec:related-work} \optional{sec:related-work}

Several approaches to incorporating graded truth degrees into the answer set programming framework have been proposed. In this paper, we have formally compared our approach to \cite{schockaert12fuzzy} and \cite{lukasiewicz06fuzzy}.
Most works consider a special form $h\leftarrow B$ where $h$ is an atom and $B$ is some formula
\cite{vojtas01fuzzy,damasio01monotonic,medina01multi-adjoint,damasio01antitonic}. 
Among them, \cite{vojtas01fuzzy,damasio01monotonic,medina01multi-adjoint} allow $B$ to be any arbitrary formula
corresponding to an increasing function whose arguments are the atoms
appearing in the formula. ~\cite{damasio01antitonic} allows $B$
to correspond to either an increasing function or a decreasing
function. ~\cite{madrid08towards} considers the normal program
syntax, i.e., each rule is of the form $l_0 \leftarrow l_1 \fand \dots
\fand l_m \fand not\ l_{m+1} \fand \dots \fand not\ l_{n}$, where each
$l_i$ is an atom or a strongly negated atom. In terms of
semantics, most of the previous works rely on the notion of an immediate
consequence operator and relate the fixpoint of this operator to the
minimal model of a positive program.
Similar to the approach ~\cite{lukasiewicz06fuzzy} has adopted, the answer set of
a positive program is defined as its minimal model, while the answer sets of a non-positive program are defined as minimal models of reducts. 
~\cite{nieuwenborgh07anintroduction} presented a semantics based on
the notion of an unfounded set.
Disjunctive fuzzy answer set programs were also studied in~\cite{blondeel14complexity}.

It is worth noting that some of the related works have discussed so-called residuated programs
\cite{vojtas01fuzzy,damasio01monotonic,medina01multi-adjoint,madrid08towards}, where each rule $h
\leftarrow B$ is assigned a weight $\theta$, and a rule is satisfied by an
interpretation $I$ if $I(h \leftarrow B) \geq \theta$. According to
\cite{damasio01monotonic}, this class of programs is able to
capture many other logic programming paradigms, such as possibilistic
logic programming, hybrid probabilistic logic programming, and generalized
annotated logic programming. Furthermore, as shown in
~\cite{damasio01monotonic}, a weighted rule $(h \leftarrow B,
\theta)$ can be simulated by $h \leftarrow B \fand \theta$, where
$(\fand, \leftarrow)$ forms an adjoint pair. Notice that a similar method was used in Theorem~\ref{thm:ystable-1stable} in relating $y$-stable models to $1$-stable models. 

It is well known in the Boolean stable model semantics that strong
negation can be represented in terms of new atoms
\cite{ferraris11stable}. 
Our adaptation in the fuzzy stable model semantics is similar to the
method from~\cite{madrid08towards}, in which the consistency of an interpretation is guaranteed by imposing the extra restriction $I(\sneg p) \leq\; \sneg I(p)$ for all atom $p$. Strong negation and consistency have also been studied in \cite{madrid11measuring,madrid09oncoherence}.

In addition to fuzzy answer set programming, there are other approaches developed to handle many-valued logic. For example, \cite{straccia06annotated} proposed a logic programming framework
where each literal is annotated by a real-valued interval. Another example is possibilistic logic \cite{dubois04possibilistic}, where each propositional symbol is associated with two real values in the interval $[0, 1]$ called the necessity degree and the possibility degree. Although
these semantics handle fuzziness based on quite different ideas, it
has been shown that these paradigms can be captured by fuzzy answer
set programs~\cite{damasio01monotonic}.

While the development of FASP solvers has not yet reached the maturity level of ASP solvers, there is an increasing interest recently.
\cite{alviano13fuzzy} presented an FASP solver based on answer set approximation operators and a translation to bilevel linear programming presented in~\cite{blondeel14complexity}. The implementation in~\cite{mushthofa14afinite} is based on a reduction of FASP to ASP. 
Independent from the work presented here, \cite{alviano15fuzzy} presented another promising FASP solver, named {\sc fasp2smt}, that uses SMT solvers based on a translation from FASP into satisfiability modulo theories. A large fragment of the language proposed in this paper can be computed by this solver. 
The input language allows $\neg_s$, $\fand_l$, $\for_l$, $\fand_m$, $\for_m$ as fuzzy operators, and rules of the form 
\[
  \i{Head} \ar \i{Body}
\]
where $\i{Head}$ is $p_1\odot\cdots\odot p_n$ where $p_i$ are atoms or numeric constants, and $\odot\in\{\fand_l, \for_l, \fand_m, \for_m\}$, and $\i{Body}$ is a nested formula formed from atoms and numeric constants using $\neg_s$, $\fand_l$, $\for_l$, $\fand_m$, $\for_m$. 

Example~\ref{ex:trust} can be computed by {\sc fasp2smt}. Assume a conflict of degree $0.1$ between Alice and Bob occurred at step $0$, no conflict occurred at step $1$, and a conflict of degree $0.5$ between Alice and Bob occurred at step $2$. Since the current version of {\tt fast2smt} \footnote{Downloaded in March 2016} does not yet support product conjunction $\fand_p$, we use \L ukasiewicz conjunction $\fand_l$ in formula $F_4$. 
In the input language of {\sc fasp2smt}, ``:-'' denote $\rar_r$, ``,'' denotes $\fand_m$, ``*'' denote $\fand_l$, ``+'' denotes $\for_l$, and ``not'' denotes $\neg_s$. Numeric constants begin with ``{\tt \#}'' symbol, and variables are capitalized.
The encoding in the input language of {\sc fasp2smt} is shown in Figure~\ref{fig:trust-fasp2smt}.

\begin{figure}
\begin{lstlisting}
% domain
person(alice).     person(bob).      person(carol).    
step(0).           step(1).          step(2).          step(3).
next(0,1).         next(1,2).        next(2,3).

% F1: trust is reflexive
trust(X,X,T) :- person(X), step(T).

% F2, F3: UEC
:- trust(X,Y,T) * distrust(X,Y,T).
:- not (trust(X,Y,T) + distrust(X,Y,T)), person(X), person(Y), step(T).

% F4, F5: transitivity of trust
% F4 modified: product t-norm is replaced by Lukasiwicz t-norm since 
%  the solver does not support it.
trust(X,Z,0) :- trust(X,Y,0) * trust(Y,Z,0), 
                person(X), person(Y), person(Z).
distrust(X,Y,0) :- not trust(X,Y,0), person(X), person(Y).

% F6, F7: inertia
trust(X,Y,T2) :- trust(X,Y,T1), not not trust(X,Y,T2), next(T1,T2).
distrust(X,Y,T2) :- distrust(X,Y,T1), not not distrust(X,Y,T2), next(T1,T2).

% F8: effect of conflict
distrust(X,Y,T2) :- conflict(X,Y,T1) + distrust(X,Y,T1), next(T1,T2).

% initial State
trust(alice,bob,0) :- #0.8.
trust(bob,carol,0) :- #0.7.

% action
conflict(alice,bob,0) :- #0.1.
conflict(alice,bob,2) :- #0.5.
\end{lstlisting}
\caption{Trust Example in the Input Language of {\sc fasp2smt}}
\label{fig:trust-fasp2smt}
\end{figure}

The command line to compute this program is simply:
\begin{verbatim}
       python fasp2smt.py trust.fasp
\end{verbatim}
Part of the output from {\sc fasp2smt} is shown below:
\begin{lstlisting}
	trust(alice,bob,0)	0.800000	(4.0/5.0)
	trust(bob,carol,0)	0.700000	(7.0/10.0)
	trust(alice,carol,0)	0.500000	(1.0/2.0)
	trust(alice,bob,1)	0.700000	(7.0/10.0)
	trust(bob,carol,1)	0.700000	(7.0/10.0)
	trust(alice,carol,1)	0.500000	(1.0/2.0)
	trust(alice,bob,2)	0.700000	(7.0/10.0)
	trust(bob,carol,2)	0.700000	(7.0/10.0)
	trust(alice,carol,2)	0.500000	(1.0/2.0)
	trust(alice,bob,3)	0.200000	(1.0/5.0)
	trust(bob,carol,3)	0.700000	(7.0/10.0)
	trust(alice,carol,3)	0.500000	(1.0/2.0)
	conflict(alice,bob,0)	0.100000	(1.0/10.0)
	conflict(alice,bob,1)	0.0		(0.0)
	conflict(alice,bob,2)	0.500000	(1.0/2.0)
\end{lstlisting}
The number following each atom is the truth value of the atom in  decimal notation, and the number in the parentheses is the truth value in fraction.

\section{Conclusion} \label{sec:conclusion} \optional{sec:conclusion}

We introduced a stable model semantics for fuzzy propositional
formulas, which generalizes both the Boolean stable model semantics
and fuzzy propositional logic. The syntax is the same as the syntax of
fuzzy propositional logic, but the semantics allows us to distinguish {\em stable
  models} from non-stable models. The formalism allows highly
configurable default reasoning involving fuzzy truth values. 
The proposed semantics, when we restrict threshold to be~$1$ and assume all atoms to be subject to minimization, is essentially equivalent to fuzzy equilibrium logic, but is much simpler. To the best of our knowledge, our representation of the commonsense law of inertia involving fuzzy values is new. The representation uses nested fuzzy operators, which are not available in other fuzzy ASP semantics with restricted syntax.

We showed that several traditional results in answer set programming can be naturally extended to this formalism, and expect that more results can be carried over. 
Also, it would be possible to generalize the semantics to the first-order level, similar to the way the Boolean stable model semantics was generalized in~\cite{ferraris11stable}.

\bibliographystyle{named}

\appendix

\section{Proofs} \label{sec:proofs}

\subsection{Proof of Proposition \ref{prop:monotone}\optional{prop:monotone}}
\begin{lemma}\label{lem:conjunction_le}
For any fuzzy conjunction $\fand$, we have $\fand(x, y) \le x$ and $\fand(x, y) \le y$.
\end{lemma}
\begin{proof}
By the conditions imposed on fuzzy conjunctions, we have $\fand(x, y)\le \fand(x, 1) = x$ and $\fand(x, y)\le \fand(1, y) = y$. 
\qed
\end{proof}

\bigskip
\noindent{\bf Proposition~\ref{prop:monotone}.\optional{prop:monotone}}\ \ 
{\sl
For any interpretations $I$ and $J$ such that $J\le^{\bf p} I$, it holds that 
\[
\u_J(F^I)\le \u_I(F).
\]
}

\begin{proof}
By induction on $F$.
\bi
\item $F$ is an atom $p$. $\u_J(F^I) = J(p)$ and $\u_I(F) = I(p)$. Clear from the assumption~$J\le^{\bf p} I$.

\item $F$ is a numeric constant $c$. Clearly, $\u_J(F^I)=c=\u_I(F)$.

\item $F$ is $\neg G$. $\u_J(F^I) = \u_J(\u_I(\neg G)) = \u_I(\neg G) = \u_I(F)$.

\item $F$ is $G\odot H$, where $\odot$ is $\fand$ or $\for$. $\u_J(F^I) = \u_J(G^I)\odot \u_J(H^I)$. By I.H. we have $\u_J(G^I)\le \u_I(G)$ and $\u_J(H^I)\le \u_I(H)$. Since $\fand$ and $\for$ are both increasing, we have $\u_J(F^I)\  =\  \u_J(G^I)\odot \u_J(H^I)\ \le\  \u_I(G)\odot \u_I(H)\ =\ \u_I(F)$.

\item $F$ is $G\rar H$. $\u_J(F^I) = (\u_J(G^I)\rar \u_J(H^I))\fand_m \u_I(F)$. By Lemma \ref{lem:conjunction_le}, $\u_J(F^I) \le \u_I(F)$.
\ei
\qed
\end{proof}

\subsection{Proof of Proposition \ref{prop:reduct-satisfaction}\optional{prop:reduct-satisfaction}}

\begin{lemma}\label{prop:reduct_val}
For any (fuzzy) formula $F$ and (fuzzy) interpretation $I$, we have $\u_I(F) = \u_I(F^I)$.
\end{lemma}

\begin{proof}
By induction on $F$.
\bi
\item $F$ is an atom $p$ or a numeric constant. Clear from the fact $F^I=F$.

\item $F$ is $\neg G$. Then we have $\u_I(F) = \u_I(\neg G) = \u_I(\u_I(\neg G)) = \u_I(F^I)$.

\item $F$ is $G\odot H$, where $\odot$ is $\fand$, $\for$, or $\rar$. Then $F^I = (G^I \odot H^I)\fand_m \u_I(F\odot G)$. By I.H., we have $\u_I(G) = \u_I(G^I)$ and $\u_I(H) = \u_I(H^I)$. So we have
\begin{align*}
 \u_I(F^I) &= min\left\{\u_I(G^I \odot H^I),\ \ \u_I(G\odot H)\right\}\\
          &= min\left\{\odot(\u_I(G^I), \u_I(H^I)),\ \  \odot(\u_I(G), \u_I(H))\right\}\\
          &= min\left\{\odot(\u_I(G), \u_I(H)),\ \  \odot(\u_I(G), \u_I(H))\right\}\\
          &= \odot(\u_I(G), \u_I(H))\\
          &= \u_I(F).
\end{align*}
\ei
\qed
\end{proof}

Proposition \ref{prop:reduct-satisfaction} is an immediate corollary to Lemma \ref{prop:reduct_val}.

\bigskip\noindent{\bf Proposition~\ref{prop:reduct-satisfaction}.\optional{prop:reduct-satisfaction}}\ \ 
{\sl
A (fuzzy) interpretation $I$ satisfies a (fuzzy) formula $F$ if and only if $I$ satisfies $F^I$.
}

\subsection{Proof of Theorem \ref{thm:ystable-1stable}\optional{thm:ystable-1stable}}

\begin{lemma}\label{lem:from_th_y_to_th_1_model}
For any fuzzy formula $F$, any interpretation $I$, and any implication $\rar$ that satisfies $\rar\!\!(x, y) = 1$ iff $y\geq x$, we have that 
$I$ is a $y$-model of $F$ iff $I$ is a $1$-model of $y\rar F$.
\end{lemma}

\begin{proof}
By definition, $I\models_y F$ means that $\u_I(F)\geq y$. 
Since $\rar\!\!(x, y) = 1$ iff $y \geq x$, 
$\u_I(F)\geq y$ iff $\rar\!\!(y, \u_I(F))=1$ iff $\u_I(y\rar F) = 1$ iff 
$I\models_1 (y\rar F)$.
\qed
\end{proof}

\bigskip
\noindent{\bf Theorem~\ref{thm:ystable-1stable}.\optional{thm:ystable-1stable}}\ \ 
{\sl
For any fuzzy formula $F$, an interpretation $I$ is a $y$-stable model of $F$ relative to ${\bf p}$ iff $I$ is a $1$-stable model of $y\rar F$ relative to ${\bf p}$ as long as the implication $\rar$ satisfies the condition $\rar\!\!(x, y) = 1$ iff $y\geq x$. 
}\medskip

\begin{proof}
\[
  \text{$I$ is a $y$-stable model of $F$ relative to ${\bf p}$} 
\]
iff 
\[
  \text{$I\models_y F$ and there is no $J<^{\bf p} I$ such that $J\models_y F^I$}
\]
iff (by Lemma \ref{lem:from_th_y_to_th_1_model})
\[
\ba c
   \text{$I\models_1 y\rar F$ and there is no $J <^{\bf p} I$ such that $J\models_1 (y\rar F^I)$}
\ea
\]
iff (since $\u_I(y\rar F)=1$)
\[
\ba c
   \text{$I\models_1 y\rar F$ and there is no  $J <^{\bf p} I$ such that $J\models_1 (y\rar F^I)\fand_m \u_I(y\rar F)$}
\ea
\]
iff (since $(y\rar F^I)\fand_m \u_I(y\rar F)= (y\rar F)^I$)
\[ 
\text{
 $I$ is a $1$-stable model of $y\rar F$ relative to ${\bf p}$.
}
\]
\qed
\end{proof}

\subsection{Proof of Theorem \ref{thm:fuzzy-sm}\optional{thm:fuzzy-sm}}

\begin{lemma}\label{lem:eq-reduct-star}\optional{lem:eq-reduct-star}
For any formula $F$ and any interpretations $I$ and $J$ such that 
$J\le^{\bf p} I$, $\u_{I\cup J^{\bf p}_{\bf q}}(F^*({\bf q})) = \u_J(F^I)$.
\end{lemma}

\begin{proof}
By induction on $F$.

\bi
\item  $F$ is a numeric constant $c$, or an atom not in ${\bf p}$. 
       Immediate from the fact that $J$ and $I$ agree on 
       $F^*({\bf q}) = F = F^I$.

\item  $F$ is an atom $p_i\in {\bf p}$. 
       $F^*({\bf q}) = q_i$ and $F^I = p_i$. 
       Clear from the fact that 
       $\u_{I\cup J^{\bf p}_{\bf q}}(q_i) = \u_J(p_i)$.

\item  $F$ is $\neg G$. 
       $\u_{I\cup J^{\bf p}_{\bf q}}((\neg G)^*({\bf q})) 
       = \u_I(\neg G) = \u_J(\u_I(\neg G))
       = \u_J(F^I)$.

\item  $F$ is $G\odot H$, where $\odot$ is $\fand$ or $\for$. 
       Immediate by I.H. on $G$ and $H$.

\item  $F$ is $G\rar H$. 
       By I.H.,  $\u_{I\cup J^{\bf p}_{\bf q}}(G^*({\bf q})) = \u_J(G^I)$ and
       $\u_{I\cup J^{\bf p}_{\bf q}}(H^*({\bf q})) = \u_J(H^I)$.
$F^*({\bf q})$ is 
  $(G^*({\bf q})\rar H^*({\bf q}))\fand_m (G\rar H)$, and 
$F^I$ is 
  $(G^I\rar H^I)\fand_m \u_I(G\rar H)$.

Then the claim is immediate from I.H. 
\end{itemize}
\qed
\end{proof}

\noindent{\bf Theorem~\ref{thm:fuzzy-sm}.\optional{thm:fuzzy-sm}}\ \ 
{\sl
A fuzzy interpretation $I$ is a fuzzy stable model of $F$ relative to ${\bf p}$ iff 
\bi 
\item  $I \models F$, and
\item  there is no fuzzy interpretation $J$ such that  $J<^{\bf p} I$ and $I\cup J^{\bf p}_{\bf q}\models F^*({\bf q})$.
\ei
}\medskip

\begin{proof}
\[
  \text{$I$ is a fuzzy stable model of $F$ relative to {\bf p}}
\]
iff 
\[
  \text{$I\models F$ and there is no $J <^{\bf p} I$ such that $J \models F^I$} 
\]
iff (by Lemma~\ref{lem:eq-reduct-star})
\[
  \text{$I\models F$ and there is no $J <^{\bf p} I$ 
   such that $I\cup J^{\bf p}_{\bf q}\models F^*({\bf q})$}.
\]
\qed
\end{proof}

\subsection{Proofs of Theorems~\ref{thm:cl-fuzzy-sm} and \ref{thm:fuzzy-cl-sm}}

The following lemma immediately follows from Lemma~\ref{lem:conjunction_le}.

\begin{lemma}\label{cor_tnorm_conjunction}
For any fuzzy conjunction $\fand$, $\fand(x, y)=1$ if and only if $x=1$ and $y=1$.
\end{lemma}

Define the mapping $\i{defuz}(I)$ that maps a fuzzy interpretation 
$I = \{(p_1, x_1), \dots,$ $(p_n, x_n)\}$ to a classical interpretation such that $\i{defuz}(I) = \{p_i\mid (p_i, 1) \in I\}$.

\begin{lemma}\label{lem:df-models}\optional{lem:df-models}
for any fuzzy interpretation $I$ and any classical propositional formula~$F$,
\bi
\item[(i)]  if $\u_I(F^{fuzzy}) = 1$, then $\i{defuz}(I)\models F$, and
\item[(ii)] if $\u_I(F^{fuzzy}) = 0$, then $\i{defuz}(I)\not\models F$.
\ei
\end{lemma}

\begin{proof}
We prove by induction on $F$.
\bi
\item $F$ is an atom $p$. (i) Suppose $I \models_1 F^{fuzzy}$. Then $\u_I(p) = 1$, and thus $p \in {\i{defuz}}(I)$. So $\i{defuz}(I) \models F$. (ii) Suppose $\u_I(F^{fuzzy}) = 0$. Then $\u_I(p) = 0$, and thus $p \notin \i{defuz}(I)$. So $\i{defuz}(I) \not\models F$. 

\item $F$ is $\bot$. (i) There is no interpretation that satisfies $F^{fuzzy}=0$ to the degree $1$. So the claim is trivially true. (ii) Since no interpretation satisfies $F$, $\i{defuz}(I) \not\models F$.

\item $F$ is $\top$. (i) All interpretations satisfy $F$. So $\i{defuz}(I) \models F$. (ii) There is no interpretation $I$ such that $\u_I(F^{fuzzy})=0$. So (ii) is trivially true.

\item $F$ is $\neg G$. Then $F^{fuzzy}$ is $\neg_s G^{fuzzy}$. (i) Suppose $I \models_1 F^{fuzzy}$. We have $\u_I(F^{fuzzy}) = 1-\u_I(G^{fuzzy}) = 1$, so $\u_I(G^{fuzzy}) =0$. By I.H., $\i{defuz}(I) \not\models G$, and thus $\i{defuz}(I) \models F$. (ii) Suppose $\u_I(F^{fuzzy}) = 0$, Then $1-\u_I(G^{fuzzy}) = 0$, $\u_I(G^{fuzzy}) = 1$, i.e., $I \models_1 \u_I(G^{fuzzy})$. By I.H., $\i{defuz}(I) \models G$, and thus $\i{defuz}(I) \not\models F$.

\item $F$ is $G \land H$. Then $F^{fuzzy}$ is $G^{fuzzy} \fand_m H^{fuzzy}$. (i) Suppose $I \models_1 F^{fuzzy}$. By Lemma~\ref{cor_tnorm_conjunction}, $\u_I(G^{fuzzy}) = \u_I(H^{fuzzy}) = 1$, i.e., $I \models_1 G^{fuzzy}$ and $I \models_1 H^{fuzzy}$. By I.H., $\i{defuz}(I) \models G$ and $\i{defuz}(I) \models H$. It follows that $\i{defuz}(I) \models G \land H = F$. (ii) Suppose $\u_I(F^{fuzzy}) = min(\u_I(G^{fuzzy}), \u_I(H^{fuzzy})) = 0$. Then $\u_I(G^{fuzzy}) = 0$ or $\u_I(H^{fuzzy}) = 0$. By I.H., $\i{defuz}(I) \not\models G$ or $\i{defuz}(I) \not\models H$. It follows that $\i{defuz}(I) \not\models G \land H = F$.

\item $F$ is $G \lor H$. Then $F^{fuzzy}$ is $G^{fuzzy} \for_m H^{fuzzy}$. (i) Suppose $I \models_1 F^{fuzzy}$, and as the disjunction is defined as $\for_m(x, y) = max(x, y)$, $\u_I(G^{fuzzy}) = 1$ or $\u_I(H^{fuzzy}) = 1$, i.e., $I \models_1 G^{fuzzy}$ or $I \models_1 H^{fuzzy}$. By I.H., $\i{defuz}(I) \models G$ or $\i{defuz}(I) \models H$. It follows that $\i{defuz}(I) \models G \lor H = F$. (ii) Suppose $\u_I(F^{fuzzy}) = max(\u_I(G^{fuzzy}), \u_I(H^{fuzzy})) = 0$. Then $\u_I(G^{fuzzy}) = 0$ and $\u_I(H^{fuzzy}) = 0$. By I.H., $\i{defuz}(I) \not\models G$ and $\i{defuz}(I) \not\models H$. It follows that $\i{defuz}(I) \not\models G \lor H = F$.

\item $F$ is $G \rar H$. Then $F^{fuzzy}$ is $G^{fuzzy} \rar_s H^{fuzzy}$. (i) Suppose $I \models_1 F^{fuzzy}$. We have $\u_I(F^{fuzzy}) = max(1-\u_I(G^{fuzzy}), \u_I(H^{fuzzy})) = 1$, so that $\u_I(G^{fuzzy}) = 0$ or $\u_I(H^{fuzzy}) = 1$.\footnote{This does not hold for an arbitrary choice of implication. For example, consider $\rar_l(x, y)=min(1-x+y, 1)$, then from $I \models_1 G \rar_l H$, we can only conclude $\u_I(H) \geq \u_I(G)$. Furthermore, under this choice of implication, the modified statement of the lemma does not hold. A counterexample is $F=\neg_s p \rar_l q$, $I = \left\{(p, 0.5), (q, 0.6)\right\}$. Clearly, $I \models_1 F^{fuzzy}$ but $\i{defuz}(I)=\emptyset \not\models F$.} By I.H., $\i{defuz}(I) \not\models G$ or $\i{defuz}(I) \models H$. It follows that $\i{defuz}(I) \models G \rar H = F$. \\
(ii) Suppose $\u_I(F^{fuzzy})=max\left\{1-\u_I(G^{fuzzy}), \u_I(H^{fuzzy})\right\} = 0$. Then $\u_I(G^{fuzzy}) = 1$ and $\u_I(H^{fuzzy}) = 0$. By I.H., $\i{defuz}(I) \models G$ and $\i{defuz}(I) \not\models H$. Therefore $\i{defuz}(I) \not\models G \rar H$.
\ei
\qed
\end{proof}

\begin{lemma}\label{lem:true-eq-1}\optional{lem:true-eq-1}
For any classical interpretation $I$ and any classical propositional formula~$F$, $I\models F$ if and only if $I^{fuzzy}\models F^{fuzzy}$.
\end{lemma}

\begin{proof}
By induction on $F$.
\qed
\end{proof}

\BOC
\begin{proof}
We prove by induction on $F$. Note that the way by which $I^{fuzzy}$ is constructed guarantees that atoms in $F^{fuzzy}$ can only be assigned $0$ or $1$ in $I^{fuzzy}$.
\begin{itemize}
\item if $F$ is an atom $p$, according to the construction of $I^{fuzzy}$, $\u_I(F)=\u_I(p)=\textbf{t}\iff p^{I^{fuzzy}}=(F^{fuzzy})^{I^{fuzzy}}=1$;
\item if $F$ is $\bot$, the two directions are both trivially true; if $F$ is $\top$, $\u_I(F)=\textbf{true}$ and $(F^{fuzzy})^{I^{fuzzy}}=1$;
\item if $F$ is $\neg G$, then $F^{fuzzy} = \neg G^{fuzzy}$. By I.H., $\u_I(G)=\textbf{true}\iff (G^{fuzzy})^{I^{fuzzy}}=1$, which says the same as $\u_I(G)=\textbf{false}\iff (G^{fuzzy})^{I^{fuzzy}} \neq 1$. Since atoms can only be assigned $0$ or $1$ in $I^{fuzzy}$, by Lemma \ref{lem_crisp_assignment}, $(G^{fuzzy})^{I^{fuzzy}} \in \left\{0, 1\right\}$. It follows that $\u_I(G)=\textbf{false}\iff (G^{fuzzy})^{I^{fuzzy}} = 0$. So $\u_I(F) = \textbf{true}$ $\iff$ $\u_I(G)=\textbf{false}$ $\iff$ $(G^{fuzzy})^{I^{fuzzy}} = 0$ $\iff$ $(F^{fuzzy})^{I^{fuzzy}}=1$;
\item if $F$ is $G\odot H$, then $F^{fuzzy}=G^{fuzzy}\odot^\prime H^{fuzzy}$, where $\odot^\prime$ is the corresponding fuzzy operator of $\odot$. By I.H., $\u_I(G)=\textbf{true}\iff (G^{fuzzy})^{I^{fuzzy}}=1$ and $\u_I(H)=\textbf{true}\iff (H^{fuzzy})^{I^{fuzzy}}=1$. Consequently, $\u_I(F) = \textbf{true}$ $\iff$ $\odot(\u_I(G), \u_I(H)) = \textbf{true}$. As atoms can only be assigned $0$ or $1$ in $I^{fuzzy}$, by Lemma \ref{lem_crisp_assignment}, $(G^{fuzzy})^{I^{fuzzy}}, (H^{fuzzy})^{I^{fuzzy}} \in \left\{0, 1\right\}$. As any valid fuzzy operator must be proper generalization of its corresponding classical operator, $\odot(\u_I(G), \u_I(H)) = \textbf{true}$ $\iff$ $\odot^\prime((G^{fuzzy})^{I^{fuzzy}}, (H^{fuzzy})^{I^{fuzzy}}) = 1$.
\end{itemize}
\qed
\end{proof}
\EOC

\bigskip
\noindent{\bf Theorem~\ref{thm:cl-fuzzy-sm}.\optional{thm:cl-fuzzy-sm}}\ \ 
{\sl
For any classical propositional formula $F$ and any classical propositional interpretation $I$, $I$ is a stable model of $F$ relative to ${\bf p}$ iff $I^\mi{fuzzy}$ is a stable model of $F^\mi{fuzzy}$ relative to ${\bf p}$.
}
\medskip

\begin{proof}
($\Rightarrow$) Suppose $I$ is a stable model of~$F$ relative to~${\bf p}$. From the fact that $I\models F$, by Lemma \ref{lem:true-eq-1}, 
$I^{fuzzy}\models F^{fuzzy}$. 

Next we show that there is no fuzzy interpretation $J <^{\bf p} I^{fuzzy}$ such that $J\models (F^{fuzzy})^I$. Suppose, for the sake of contradiction, that there exists such $J$. Since 
$J\models (F^{fuzzy})^I$, or equivalently, $J\models (F^I)^{fuzzy}$, by Lemma~\ref{lem:df-models}, we get $\i{defuz}(J)\models F^I$. 

Since $J<^{\bf p} I^{fuzzy}$, $J$ and $I^{fuzzy}$ agree on all atoms not in ${\bf p}$. Since $I^{fuzzy}$ assigns either $0$ or $1$ to each atom, it follows that $J$, as well as $\i{defuz}(J)$, assigns the same truth values as $I^{fuzzy}$ to atoms not in ${\bf p}$.
The construction of $\i{defuz}(J)$ guarantees that $\i{defuz}(J)^{fuzzy} \leq^{\bf p} J <^{\bf p} I^{fuzzy}$. Since both $\i{defuz}(J)^{fuzzy}$ and $I^{fuzzy}$ assigns either $0$ or $1$ to each atom, there is at least one atom $p \in {\bf p}$ such that $\u_{{defuz}(J)^{fuzzy}}(p) = 0$ and $\u_{I^{fuzzy}}(p) = 1$, and consequently ${defuz}(J)(p) = \false$ and $I(p) = \true$. 
So $\i{defuz}(J) <^{\bf p} I$, and together with the fact that $\i{defuz}(J)\models F^I$, it follows that $I$ is not a stable model of $F$, which is a contradiction. Thus, there is no such $J$, from which we conclude that $I^{fuzzy}$ is a stable model of $F^{fuzzy}$ relative to ${\bf p}$.

\medskip\noindent
($\Leftarrow$) Suppose $I^{fuzzy}$ is a stable model of a fuzzy formula $F^{fuzzy}$ relative to ${\bf p}$. Then $I^{fuzzy} \models F^{fuzzy}$. By Lemma \ref{lem:true-eq-1}, $I\models F$. 

Next we show there is no $J<^{\bf p} I$ such that $J\models F^I$. Suppose, for the sake of contradiction, that there exists such $J$. Then by Lemma~\ref{lem:true-eq-1}, $J^{fuzzy}\models (F^I)^{fuzzy}$, and obviously $J^{fuzzy} <^{\bf p} I^{fuzzy}$. It follows that $I^{fuzzy}$ is not a stable model of $F^{fuzzy}$ relative to ${\bf p}$, which is a contradiction. So there is no such $J$, from which we conclude that $I$ is a stable model of $F$ relative to ${\bf p}$.
\qed
\end{proof}

\bigskip
The proof of Theorem~\ref{thm:fuzzy-cl-sm} is the same as the right-to-left direction of the proof of Theorem~\ref{thm:cl-fuzzy-sm}. Notice that the left-to-right direction of the proof of Theorem~\ref{thm:cl-fuzzy-sm} relies on Lemma~\ref{lem:df-models}, which assumes the particular selection of fuzzy operators, so this direction does not apply to the setting of Theorem~\ref{thm:fuzzy-cl-sm}.

\BOC
\subsection{Proof of Theorem \ref{thm:fuzzy-cl-sm}}

\noindent{\bf Theorem~\ref{thm:fuzzy-cl-sm}.\optional{thm:fuzzy-cl-sm}}\ \ 
{\sl
For any classical propositional formula $F$, let $F_1^\mi{fuzzy}$ be the fuzzy formula obtained from $F$ by replacing $\bot$ with ${0}$,  $\top$ with ${1}$, $\neg$ with any fuzzy negation symbol, $\land$ with any fuzzy conjunction symbol, $\lor$ with any fuzzy disjunction symbol, and $\rar$ with any fuzzy implication symbol. For any classical propositional interpretation $I$, if $I^\mi{fuzzy}$ is a fuzzy stable model of $F_1^\mi{fuzzy}$ relative to ${\bf p}$, then $I$ is a Boolean stable model of $F$ relative to ${\bf p}$.
}

\vspace{3 mm}

\begin{proof}
Suppose $I^{fuzzy}$ is a $1$-stable model of a fuzzy formula $F_1^{fuzzy}$ relative to ${\bf p}$. Then $I^{fuzzy} \models_1 F_1^{fuzzy}$. By Lemma \ref{lem:true-eq-1}, $I \models F$. Next we show there does not exist $J <^{\bf p} I$ such that $J \models F_1^{{I}}$. Suppose, to the contrary, that there exists such $J$. Then by Lemma \ref{lem:true-eq-1}, $J^{fuzzy} \models_1 (F_1^{{I}})^{fuzzy}$, and obviously $J^{fuzzy} <^{\bf p} I^{fuzzy}$. It follows that $I^{fuzzy}$ cannot be a stable model of $F_1^{fuzzy}$, which is a contradiction. So there does not exist such $J$. Therefore, $I$ is a stable model of $F$ relative to ${\bf p}$.
\qed
\end{proof}
\EOC

\subsection{Proof of Theorem~\ref{thm:normal-fuzzysm}\optional{thm:normal-fuzzysm}}

\noindent{\bf Theorem~\ref{thm:normal-fuzzysm}.\optional{thm:normal-fuzzysm}}\ \ 
{\sl
For any normal FASP program $\Pi=\{r_1, \dots, r_n\}$, let $F$ be the fuzzy formula $r_1\fand\dots\fand r_n$, where $\fand$ is any fuzzy conjunction. An interpretation $I$ is a fuzzy answer set of $\Pi$ in the sense of \cite{lukasiewicz06fuzzy} if and only if $I$ is a stable model of $F$. 
}
\medskip

\begin{proof}
By $\Pi^\mu{I}$ we denote the reduct of $\Pi$ relative to $I$ as defined in~\cite{lukasiewicz06fuzzy}, which is also reviewed in Section~\ref{ssec:normal-fasp}.

\medskip\noindent
($\Rightarrow$) Suppose $I$ is a fuzzy answer set of $\Pi$. By definition, $I \models\Pi$, and thus $\u_I(r_i)=1$ for all $r_i\in \Pi$. So 
\[
  \u_I(F) = \u_I(r_1\fand\dots\fand r_n) = 1,
\]
i.e., $I\models F$. 

Next we show that there is no $J <^\sigma I$ such that $J\models F^I$, where $\sigma$ is the underlying signature. For each 
$r_i =   a\ \ar_r\ b_1 \fand \dots \fand b_m \fand \neg b_{m+1}\fand \dots \fand \neg b_n$,
\[ 
r^{{I}}_i = \u_I(r_i) \fand_m (a \leftarrow_r b_1 \fand \dots \fand b_m \fand \u_I(\neg_s b_{m+1}) \fand \dots \fand \u_I(\neg_s b_n)).
\]
Suppose, for the sake of contradiction, that there exists an interpretation $J<^\sigma I$ such that $J \models F^I$. Then, for all $r_i \in \Pi$, $J \models r^{{I}}_i$, i.e.,
\begin{equation*}
J \models \u_I(r_i) \fand_m (a \leftarrow_r b_1 \fand \dots \fand b_m \fand \u_I(\neg_s b_{m+1}) \fand \dots \fand \u_I(\neg_s b_n)).
\end{equation*}
It follows that
\begin{equation*}
  J \models a \leftarrow_r b_1 \fand \dots \fand b_m \fand \u_I(\neg_s b_{m+1}) \fand \dots \fand \u_I(\neg_s b_n).
\end{equation*}
So $J \models \Pi^\mu{I}$. Together with the fact $J <^\sigma I$, this contradicts that $I$ is a fuzzy answer set of $\Pi$. So there is no such $J$, from which we conclue that $I$ is a stable model of~$F$. 

\medskip\noindent
($\Leftarrow$) Suppose $I$ is a stable model of $F$. From the fact that $I\models F$, it follows that $\u_I(r_i)=1$ for all $r_i\in \Pi$. Thus $I\models\Pi$.

Next we show that there is no $J <^\sigma I$ such that $J\models \Pi^\mu{I}$,  where $\sigma$ is the underlying signature. The reduct $\Pi^\mu{I}$ contains the following rule for each original rule $r_i \in \Pi$:
\[
a \leftarrow b_1 \fand \dots \fand b_m \fand \u_I(\neg b_{m+1}) \fand \dots \fand \u_I(\neg b_n).
\]
Suppose, for the sake of contradiction, that there exists $J <^\sigma I$ such that $J \models \Pi^\mu{I}$. Then for each rule $r_i \in \Pi$,
\[ 
  J \models a \leftarrow_r b_1 \fand \dots \fand b_m \fand \u_I(\neg_s b_{m+1}) \fand \dots \fand \u_I(\neg_s b_n).
\]
As $I\models\Pi$, for all $r_i\in\Pi$, $\u_I(r_i)=1$, so
\[
   J\models \u_I(r_i) \fand_m (a \leftarrow_r b_1 \fand \dots \fand b_m \fand \neg_s \u_I(b_{m+1}) \fand \dots \fand \u_I(\neg_s b_n))
\]
or equivalently, $J\models r_i^I$ for each $r_i\in\Pi$. 
Therefore, $J \models (r_1\fand\dots\fand r_n)^I$, i.e., $J\models F^I$. Together with the fact $J <^\sigma I$, this contradicts that $I$ is a stable model of $F$. So there is no such $J$, from which we conclude that $I$ is a fuzzy answer set of $\Pi$.
\qed
\end{proof}

\subsection{Proof of Theorem \ref{thm:equil-sm-nostrneg}\optional{thm:equil-sm-nostrneg}}

Let $\sigma$ be a signature, and let $F$ be a fuzzy formula of signature $\sigma$ that does not contain strong negation.

For any two interpretations $I$ and $J$ such that $J\le^{\sigma} I$, we define the fuzzy N5 valuation $\V_{J,I}$ as follows. For every atom $p\in\sigma$,
\begin{itemize}
\item  $\V_{J, I}(h, p) = \left[\u_J(p), 1\right]$, and
\item  $\V_{J, I}(t, p) = \left[\u_I(p), 1\right]$.
\end{itemize}
It can be seen that $\V_{J,I}$ is always a valid valuation as long as $J \leq^{\sigma} I$. Clearly, $\I_{\V_{I, I}} = I$ for any interpretation $I$.

\begin{lemma}\label{lem:model-tw}\optional{lem:model-tw}
For any interpretations $I$ and $J$ such that $J\le^{\sigma} I$, it holds that $v_I(F) = \V_{J,I}^-(t, F)$.
\end{lemma}

\begin{proof}
By induction on $F$.
 
\begin{itemize}
\item  $F$ is an atom $p$. Then $\u_I(F) = \u_I(p) = \V_{J, I}^-(t, p) = \V_{J, I}^-(t, F)$.

\item $F$ is a numeric constant $c$. Then $\u_I(F) = c = \V_{J, I}^-(t, c) = \V_{J, I}^-(t, F)$.

\item $F$ is $\neg G$. Then $\u_I(F) = \neg(\u_I(G))$. By I.H., $\u_I(F) = \neg(\V_{J, I}^-(t, G)) = \V_{J, I}^-(t, F)$.

\item $F$ is $G\odot H$, where $\odot$ is $\fand$, $\for$ or $\rar$. Then $\u_I(F) = \u_I(G)\odot \u_I(H)$. 
By I.H., $\u_I(F)= \V_{J, I}^-(t, G)\odot \V_{J, I}^-(t, H)  = \V_{J, I}^-(t, F)$.
\end{itemize}
\qed
\end{proof}

The following is a corollary to Lemma \ref{lem:model-tw}.

\begin{cor}\label{model_eq}\optional{model-eq}
$I \models F$ if and only if $\V_{I, I}$ is a model of $F$.
\end{cor}

\begin{proof}
By Lemma \ref{lem:model-tw}, $\u_I(F) = \V_{I, I}^-(t, F)$. Furthermore, 
$\V_{I, I}^-(h, F) = \V_{I, I}^-(t, F)$  can be proven by induction.
\qed
\end{proof}

\begin{lemma}\label{lem:hlb-star}\optional{lem:hlb-star}
For any interpretations $I$ and $J$ such that $J\le^{\sigma} I$, it holds that $\u_J(F^I) = \V_{J,I}^-(h, F)$.
\end{lemma}

\begin{proof}
By induction on $F$.

\begin{itemize}
\item  $F$ is an atom $p$ or a numeric constant. Then $\u_J(F^I) = \u_J(F) = \V_{J, I}^-(h, F)$. 


\item  $F$ is $\neg G$. Then $\u_J(F^I) = \neg \u_I(G)$, 
       and by Lemma~\ref{lem:model-tw}, 
       $\neg \u_I(G) = \neg \V_{J,I}^-(t,G) = \V_{J,I}^-(h,F)$.


\item  $F$ is $G\odot H$, where $\odot$ is $\fand$ or $\for$. 
       Then 
\[
\ba {rl}
\u_J(F^{{I}}) = & \u_J(G^{{I}}\odot H^{{I}}) = \u_J(G^{{I}})\odot \u_J(H^{{I}}) \\
= & (\text{by I.H.}) \\
  & \V_{J, I}^-(h, G)\odot \V_{J, I}^-(h, H)  \\
= & \V_{J, I}^-(h, G\odot H) = \V_{J, I}^-(h, F).
\ea
\]

\item  $F$ is $G\rar H$. Then 
\[
\ba {rl}
\u_J(F^I) =& \u_J((G^I\rar H^I)\fand_m \u_I(G\rar H))\\
= & min(\u_J(G^I\rar H^I), \u_I(G\rar H)) \\
= & (\text{by I.H. and Lemma~\ref{lem:model-tw}}) \\
  & min(\V_{J, I}^-(h, G)\rar \V_{J, I}^-(h, H), \V_{J, I}^-(t, G) \rar \V_{J, I}^-(t, H))\\
= & \V_{J, I}^-(h, G\rar H)=\V_{J, I}^-(h, F).
\ea
\]

\end{itemize}
\qed
\end{proof}

The following is a corollary to Lemma \ref{lem:hlb-star}.

\begin{cor}\label{cor:star-valuation}\optional{cor:star-valuation}
For any interpretations $I$ and $J$ such that $J \leq^{\sigma} I$, $J\models F^I$ if and only if $\V_{J, I}$ is a model of $F$.
\end{cor}

\begin{lemma}\label{lessthan_eq}
For two interpretations $I$ and $J$, it holds that $\V_{J, I} \prec \V_{I, I}$ if and only if $J < I$.
\end{lemma}

\begin{proof}
($\Rightarrow$) Suppose $\V_{J, I}\prec \V_{I, I}$. For every atom $a$, $\V_{I, I}(h, a) \subseteq \V_{J, I}(h, a)$, which means $\V_{J, I}^-(h, a) \leq \V_{I, I}^-(h, a)$. So $J \leq I$. 
Furthermore, there is at least one atom $p$ satisfying $\V_{I, I}(h, p) \subset \V_{J, I}(h, p)$. By the definition of $\V_{I, I}$ and $\V_{J, I}$, $\V_{I, I}^+(h, a)=\V_{J, I}^+(h, a) = 1$ for all $a$. Consequently, $\V_{J, I}^-(h, p)<\V_{I, I}^-(h, p)$, which means $\u_J(p) < \u_I(p)$. So $J < I$.

($\Leftarrow$) Suppose $J < I$. Then for every atom $a$, $\u_J(a) \leq \u_I(a)$. It follows that $\V_{J, I}^-(h, a) \leq \V_{I, I}^-(h, a)$ for all $a$, and as $\V_{I, I}^+(h, a)=\V_{J, I}^+(h, a) = 1$ for all $a$, we conclude $\V_{I, I}(h, a) \subseteq \V_{J, I}(h, a)$. 
Clearly, $\V_{I,I}(t,a) = \V_{J,I}(t,a)$ by definition.
Furthermore there is at least one atom $p$ such that $\u_J(p) < \u_I(p)$, i.e., $\V_{J, I}^-(h, a) < \V_{I, I}^-(h, a)$. So $\V_{J, I} \prec \V_{I, I}$.
\qed
\end{proof}

\begin{lemma}\label{lem:sm-to-equil}\optional{lem:sm-to-equil}
An interpretation $I$ is a stable model of $F$ if and only if $\V_{I, I}$ is a fuzzy equilibrium model of~$F$.
\end{lemma}

\begin{proof}
($\Rightarrow$) Suppose $I$ is a stable model of $F$. As $I$ is a model of $F$, by Corollary~\ref{model_eq}, $\V_{I, I}$ is a model of $F$. Next we show that there is no model $V'$ of $F$ such that $V'\prec \V_{I, I}$. 
Suppose, for the sake of contradiction, that there exists such $V^\prime$. Define an interpretation $J$ as $\u_J(a) = V^{\prime -}(h, a)$ for all atoms $a$. Obviously $V^\prime = \V_{J, I}$. As $V^\prime= \V_{J, I}$ is a model of $F$, by Corollary~\ref{cor:star-valuation}, $J \models F^I$. Furthermore, as $V^\prime=\V_{J, I} \prec \V_{I, I}$, by Lemma \ref{lessthan_eq}, $J < I$. Consequently, $I$ is not a $1$-stable model of $F$, which is a contradiction. So there does not exist such $V^\prime$, from which we conclude that $\V_{I, I}$ is an h-minimal model of $F$. Clearly, $\V_{I, I}(h, a) = \V_{I, I}(t, a)$ for all atoms $a$. So $\V_{I, I}$ is an equilibrium model of $F$.

\medskip
($\Leftarrow$) Suppose $\V_{I, I}$ is an equilibrium model of $F$. As $\V_{I, I}$ is a model of $F$, by Corollary~\ref{model_eq}, $I \models F$. 
Next we show that there is no $J<^\sigma I$ such that $J \models F^I$. Suppose, for the sake of contradiction, that there exists such $J$. Then by Corollary~\ref{cor:star-valuation}, the valuation $\V_{J, I}$ is a model of $F$. Furthermore, by Lemma \ref{lessthan_eq}, $\V_{J, I}\prec \V_{I, I}$. Consequently, $\V_{I, I}$ is not an h-minimal model of $F$, which contradicts that $\V_{I, I}$ is an equilibrium model of $F$. So there does not exist such $J$, from which we conclude that $I$ is a stable model of $F$. 
\qed
\end{proof}


\BOCC
\begin{lemma}\label{lem:lb-only}\optional{lem:lb-only}
Let $F$ be a fuzzy formula containing no strong negation, let $V$, $V'$ be valuations such that $V'^-(w, a)=V^-(w, a)$ and $V'^+(w, a)=1$ for all atoms $a$, where $w\in\{h, t\}$. Then $V$ is a model of $F$ iff $V'$ is a model of $F$.
\end{lemma}

\begin{proof}
We show by induction that $V^-(w, F)=V'^-(w, F)$, where $w\in\{h, t\}$.

\begin{itemize}
\item $F$ is an atom $p$. $V^-(w, F)=V^-(w, p) = V'^-(w, p) = V^-(w, F)$.

\item $F$ is a numeric constant $c$. Clearly, $V^-(w, F) = c = V'^-(w, F)$.

\item $F$ is $\neg G$. 
   By I.H, $V^-(w, G)=V'^-(w, G)$.
\[ 
   V^-(w,F) = 1-V^-(t,G) = 1-V'^-(t,G) = V'^-(w, F).
\]

\item $F$ is $G\odot H$ where $\odot$ is $\fand$ or $\for$. By I.H, $V^-(w, G)=V'^-(w, G)$ and $V^-(w, H)=V'^-(w, H)$. So
\[
\ba {rl}
  V^-(w, F) = & V^-(w, G) \odot V^-(w, H)\\
 = & V'^-(w, G) \odot V'^-(w, H)\\
 = & V'^-(w, F).
\ea
\]
\item  $F$ is $G\rar H$. 
   By I.H, $V^-(w, G)=V'^-(w, G)$ and $V^-(w, H)=V'^-(w, H)$. So
\[
\ba {rl}
V^-(h, F) = & min((V^-(h, G)\rar V^-(h, H)),\ (V^-(t, G)\rar V^-(t, H)))\\
 = & min((V'^-(h, G)\rar V'^-(h, H)),\ (V'^-(t, G)\rar V'^{-}(t, H))) \\
 = & V'^-(h, F). 
\ea
\]
And
\[
\ba {rl}
V^-(t, F) = & (V^-(t, G)\rar V^-(t, H))\\
 = &  (V'^{-}(t, G)\rar V'^{-}(t, H))\\
 = & V'^{-}(t, F).
\ea
\]
\end{itemize}

So $V^-(h, F)=V'^{-}(h, F)$ and thus $V^-(h, F)=1$ if and only if $V'^{-}(h, F)=1$, i.e, $V$ is a model of $F$ if and only if $V'$ is a model of $F$.
\qed
\end{proof}

The following is a corollary to Lemma \ref{lem:lb-only}.

\begin{cor}\label{cor:lb_only}
For any two valuations $V_1$ and $V_2$ such that $V_1^-(w, a)=V_2^-(w, a)$($w \in \left\{h, t\right\}$) for all atoms $a$, and any formula $F$ containing no strong negation, $V_1$ is a model of $F$ iff $V_2$ is a model of $F$.
\end{cor}

\begin{proof}
let $V_1'$ be valuations such that $V_1'^-(w, a)=V_1^-(w, a)$ and $V_1'^+(w, a)=1$ for all atoms $a$, and 
let $V_2'$ be valuations such that $V_2'^-(w, a)=V_2^-(w, a)$ and $V_2'^+(w, a)=1$ for all atoms $a$.
By Lemma \ref{lem:lb-only}, $V_1$ is a model of $F$ iff $V_1^\prime$ is a model of $F$, and $V_2$ is a model of $F$ iff $V_2^\prime$ is a model of $F$. Since $V_1^-(w, a)=V_2^-(w, a)$, $V_1^\prime = V_2^\prime$. So $V_1$ is a model of $F$ iff $V_2$ is a model of $F$.
\qed
\end{proof}
\EOCC

\begin{lemma}\label{lem:lb_only}\optional{lem:lb-only}
For any two valuations $V$ and $V^\prime$ such that $V^-(w, a)=V^{\prime -}(w, a)$($w \in \left\{h, t\right\}$) for all atoms $a$, and any formula $F$ containing no strong negation, $V$ is a model of $F$ iff $V^\prime$ is a model of $F$.
\end{lemma}

\begin{proof}
We show by induction that $V^-(w, F)=V'^-(w, F)$, where $w\in\{h, t\}$.

\begin{itemize}
\item $F$ is an atom $p$. $V^-(w, F)=V^-(w, p) = V'^-(w, p) = V^-(w, F)$.

\item $F$ is a numeric constant $c$. Clearly, $V^-(w, F) = c = V'^-(w, F)$.

\item $F$ is $\neg G$. 
   By I.H, $V^-(w, G)=V'^-(w, G)$.
\[ 
   V^-(w,F) = 1-V^-(t,G) = 1-V'^-(t,G) = V'^-(w, F).
\]

\item $F$ is $G\odot H$ where $\odot$ is $\fand$ or $\for$. By I.H, $V^-(w, G)=V'^-(w, G)$ and $V^-(w, H)=V'^-(w, H)$. So
\[
\ba {rl}
  V^-(w, F) = & V^-(w, G) \odot V^-(w, H)\\
 = & V'^-(w, G) \odot V'^-(w, H)\\
 = & V'^-(w, F).
\ea
\]
\item  $F$ is $G\rar H$. 
   By I.H, $V^-(w, G)=V'^-(w, G)$ and $V^-(w, H)=V'^-(w, H)$. So
\[
\ba {rl}
V^-(h, F) = & min((V^-(h, G)\rar V^-(h, H)),\ (V^-(t, G)\rar V^-(t, H)))\\
 = & min((V'^-(h, G)\rar V'^-(h, H)),\ (V'^-(t, G)\rar V'^{-}(t, H))) \\
 = & V'^-(h, F). 
\ea
\]
And
\[
\ba {rl}
V^-(t, F) = & (V^-(t, G)\rar V^-(t, H))\\
 = &  (V'^{-}(t, G)\rar V'^{-}(t, H))\\
 = & V'^{-}(t, F).
\ea
\]
\end{itemize}

So $V^-(h, F)=V'^{-}(h, F)$ and thus $V^-(h, F)=1$ if and only if $V'^{-}(h, F)=1$, i.e, $V$ is a model of $F$ if and only if $V'$ is a model of $F$.
\qed
\end{proof}

\bigskip\noindent{\bf Lemma~\ref{thm_ub_1}.}\ \ 
{\sl 
Given a formula $F$ containing no strong negation, any equilibrium model $V$ of $F$ satisfies $V^+(h, a)=V^+(t, a)=1$ for all atoms $a$.
}
\medskip

\begin{proof}
Assume that $V$ is an equilibrium model of $F$. It follows that $V^+(h, a) = V^+(t, a)$. Furthermore, for the sake of contradiction, assume that $V^+(h, a) = V^+(t, a) = v < 1$. Define $V^{\prime \prime}$ as $V^{\prime \prime -}(w, a)=V^-(w, a)$, $V^{\prime \prime +}(t, a) = V^+(t, a)$ and $V^{\prime \prime +}(h, a) = v^\prime$ where $v^\prime \in (v, 1]$. Clearly, $V^{\prime \prime} \prec V$ and by Lemma \ref{lem:lb_only}, $V^{\prime \prime}$
is a model of $F$. So $V$ is not an h-minimal model of $F$, which contradicts the assumption that $V$ is an equilibrium model of $F$. Therefore, there does not exist such $V$.
\qed
\end{proof}

\noindent{\bf Theorem~\ref{thm:equil-sm-nostrneg}.\optional{thm:equil-sm-nostrneg}}\ \ 
{\sl
Let $F$ be a fuzzy propositional formula of $\sigma$ that contains no
strong negation.
\begin{itemize}
\item[(a)]  A valuation $V$ of~$\sigma$ is a fuzzy equilibrium model
  of $F$ iff $V^-(h,p)=V^-(t,p)$, $V^+(h,p)=V^+(t,p)=1$ for all
  atoms $p$ in $\sigma$ and $\I_V$ is a stable model of $F$ relative to
  $\sigma$. 

\item[(b)] An interpretation $I$ of~$\sigma$ is a stable model of
  $F$ relative to $\sigma$ iff $I = \I_V$ for some fuzzy equilibrium
  model $V$ of $F$. 
\end{itemize}
}
\medskip

\begin{proof}
  (a) ($\Rightarrow$) Suppose $V$ is an equilibrium model of $F$. By Lemma \ref{thm_ub_1}, $V^+(h,a)=V^+(t,a)=1$ for all atoms $a$. And by the definition of fuzzy equilibrium model, $V^-(h,a)=V^-(t,a)$. It can be seen that $\V_{\I_V, \I_V} = V$. Since $\V_{\I_V, \I_V}$ is an equilibrium model of $F$, by Lemma \ref{lem:sm-to-equil}, $\I_V$ is a stable model of $F$ relative to $\sigma$.

\medskip\noindent
($\Leftarrow$) Suppose $V^+(h,a)=V^+(t,a)=1$, $V^-(h,a)=V^-(t,a)$ for all atoms $a$, and $\I_V$ is a stable model of $F$ relative to $\sigma$. By Lemma \ref{lem:sm-to-equil}, $\V_{\I_V, \I_V}$ is an equilibrium model of $F$. Since $V^+(h,a)=V^+(t,a)=1$, $V^-(h,a)=V^-(t,a)$ for all atoms $a$, it holds that $\V_{\I_V, \I_V} = V$. So $V$ is an equilibrium model of $F$.

\bigskip\noindent
(b) ($\Rightarrow$) Suppose $I$ is a stable model of $F$ relative to $\sigma$. By Lemma \ref{lem:sm-to-equil}, $\V_{I, I}$ is an equilibrium model of $F$. It can be seen that $I=\I_{\V_{I, I}}$.

\medskip\noindent
($\Leftarrow$) Take any fuzzy equilibrium model $V$ of $F$ and let $I=\I_V$.
By the definition of fuzzy equilibrium model and Lemma \ref{thm_ub_1}, $V^-(h,a)=V^-(t,a)$ and $V^+(h,a)=V^+(t,a)=1$ for all atoms $a$, which means $\V_{\I_V, \I_V} = V$. So $\V_{\I_V, \I_V}$ is an equilibrium model of $F$. By Lemma \ref{lem:sm-to-equil}, $\I_V$ is a stable model, so $I$ is a stable model.
\qed
\end{proof}

\subsection{Proof of Theorem \ref{thm:equil-sm}\optional{thm:equil-sm}}

\begin{lemma}\label{lem:val_eq_strong_neg}
For any fuzzy formula $F$ of signature~$\sigma$ that may contain strong negation and for any valuation $V$, it holds that 
$V^-(w, F) = nneg(V)^-(w, nneg(F))$.
\end{lemma}

\begin{proof}
First we show by induction that $V^-(w, F) = nneg(V)^{-}(w, F')$, where $F'$ is defined as in Section~\ref{sssec:strneg}.

\begin{itemize}
\item $F$ is an atom $p$ in $\sigma$. Clear.

\item $F$ is $\sneg p$, where $p$ is an atom in $\sigma$. Then $F'$ is $np$.
\begin{align*}
  V^-(w, F) & =  V^-(w, \sneg p) = 1-V^+(w, p) \\ 
            & = nneg(V)^-(w, np) =  nneg(V)^-(w, F').
\end{align*}
\item  $F$ is a numeric constant $c$. Clear.


\item $F$ is $\neg G$. By I.H., $V^-(w, G) = nneg(V)^-(w, G')$.
\begin{align*}
V^-(w, F) &= V^-(w, \neg G)\\
          &= 1 - V^-(t, G)\\
          &= 1 - nneg(V)^-(t, G') \\
          &= nneg(V)^-(w, \neg G')\\
          &= nneg(V)^-(w, F').
\end{align*}

\item $F$ is $G \odot H$, where $\odot$ is $\fand$ or $\for$. 
By I.H., $V^-(w, G) = nneg(V)^-(w, G')$ and 
         $V^-(w, H) = nneg(V)^-(w, H')$.
\begin{align*}
V^-(w, F) &= V^-(w, G \odot H)\\
          &= \odot(V^-(w, G), V^-(w, H))\\
          &= \odot(nneg(V)^-(w, G'), nneg(V)^{-}(w, H'))\\
          &= nneg(V)^-(w, G' \odot H')\\
          &= nneg(V)^-(w, F').
\end{align*}

\item $F$ is $G\rar H$. By I.H., $V^-(w, G) = nneg(V)^-(w, G')$ and $V^-(w, H) = nneg(V)^-(w, H')$.
\begin{align*}
V^-(h, F) &= V^-(h, G\rar H)\\
          &= min(\rar\!\!(V^-(h, G), V^-(h, H)), \rar\!\!(V^-(t, G), V^-(t, H)))\\
          &= min(\rar\!\!(nneg(V)^{-}(h, G'), nneg(V)^{-}(h, H')), \\
          & ~~~~~~~~~~ \rar\!\!(nneg(V)^{-}(t, G'), nneg(V)^{-}(t, H')))\\
         &= nneg(V)^{-}(h, G' \rar H')\\
         &= nneg(V)^{-}(h, F').
\end{align*}
And
\begin{align*}
V^-(t, F) &= V^-(t, G \rar H)\\
          &=\  \rar\!\!(V^-(t, G), V^-(t, H))\\
          &=\   \rar\!\!(nneg(V)^{-}(t, G'), nneg(V)^{-}(t, H'))\\
          &=   nneg(V)^{-}(t, F').
\end{align*}
\end{itemize}

Now notice that, for any valuation $V$, it must be the case that for all atoms $p\in\sigma$, $V^-(w, p)\le V^+(w, p)$, i.e, $V^{-}(w, p) + 1 - V^{+}(w, p) \leq 1$. It follows that 
\[  
   nneg(V)^-(w, p) + nneg(V)^-(w, np) \leq 1.
\]
Therefore, for all atoms $p$, 
\begin{align*}
nneg(V)^{-}(w, \neg_s(p \fand_l np)) &= 1 - nneg(V)^{-}(t, (p \fand_l np)) \\
          &= 1 - \fand_l(nneg(V)^{-}(t, p), nneg(V)^{-}(t, np)) \\
          &= 1 - max(nneg(V)^{-}(t, p) + nneg(V)^{-}(t, np) - 1, 0) \\
          &= 1 - 0 = 1.
\end{align*}
It follows that
\begin{align*}
V^-(w, F) &= nneg(V)^-(w, F') \\
          &= \fand_m(nneg(V)^-(w, F'), 1) \\
          &= \fand_m(nneg(V)^-(w, F'), nneg(V)^-(w, \underset{p\in \sigma}{\fand_m}\neg_s(p \fand_l np))) \\
          &= nneg(V)^-(w, F' \fand_m \underset{p\in \sigma}{\fand_m}\neg_s(p \fand_l np)) \\
          &= nneg(V)^-(w, nneg(F)).
\end{align*}
\qed
\end{proof}

\begin{cor}\label{cor:model_strong_neg}
For any fuzzy formula $F$ that may contain strong negation, a valuation $V$ is a model of $F$ iff $nneg(V)$ is a model of $nneg(F)$.
\end{cor}

\begin{proof}
By Lemma \ref{lem:val_eq_strong_neg}, $V^-(h, F) = nneg(V)^{-}(h, nneg(F))$. So $V^-(h, F) = 1$ iff $nneg(V)^{-}(h, nneg(F)) = 1$, i.e., $V$ is a model of $F$ iff $nneg(V)$ is a model of $nneg(F)$.
\qed
\end{proof}

\begin{lemma}\label{lem:ht_eq}
for all atoms $a \in \sigma$, $V(h, a)=V(t, a)$ iff 
\[
  nneg(V)(h, a)=nneg(V)(t, a) \text{ and } nneg(V)(h, na)=nneg(V)(t, na).
\]
\end{lemma}

\begin{proof}
For all atoms $a \in \sigma$,
\begin{align*}
     & V(h, a)=V(t, a) \\
\text{iff~~~} & V^-(h, a)=V^-(t, a) \text{ and } V^+(h, a)=V^+(t, a)  \\
\text{iff~~~} & nneg(V)(h, a)=nneg(V)(t, a) \text{ and } 1 - V^+(h, a)= 1 - V^+(t, a)\\
\text{iff~~~} & nneg(V)(h, a)=nneg(V)(t, a) \text{ and } nneg(V)(h, na)=nneg(V)(t, na).
\end{align*}
\qed
\end{proof}

\begin{lemma}\label{lem:prec_eq}
For two valuations $V$ and $V_1$ of signature $\sigma$, it holds that $V_1\prec V$ iff $nneg(V_1)\prec nneg(V)$.
\end{lemma}

\begin{proof}
\begin{align*}
 & V_1\prec V \\
\text{iff~~~} 
  & \text{for all atoms $p$ (from $\sigma$), $V(t, p)=V_1(t, p)$, 
          $V(h, p)\subseteq V_1(h, p)$ and}\\
  & \text{there exists an atom $a$ (from $\sigma$) such that  $V(h, a)\subset V_1(h, a)$} \\
\text{iff~~~} 
  & \text{for all atoms $p$, $V^-(t, p)=V_1^-(t, p)$, $V^+(t, p)=V_1^+(t, p)$,}\\
  & \hspace{2.8cm} \text{$V_1^-(h, p)\le V^-(h, p)$, $V_1^+(h, p)\ge V^+(h, p)$, and}\\
  & \text{there exists an atom $a$ such that  $V_1^-(h, a) < V^-(h, a)$ or   
          $V_1^+(h, a) > V^+(h, a)$}  \\
\text{iff~~~} 
  & \text{for all atoms $p$, $V^-(t, p)=V_1^-(t, p)$, $V^+(t, p)=V_1^+(t, p)$,}\\
  & \hspace{2.8cm}\text{$V_1^-(h, p)\le V^-(h, p)$, $V_1^+(h, p)\ge V^+(h, p)$, and}\\
  & \text{there exists an atom $a$ such that  $V_1^-(h, a) < V^-(h, a)$ or} \\   
  & \hspace{2.8cm}\text{$1- V_1^+(h, a) < 1- V^+(h, a)$}  \\
\text{iff~~~} 
  & \text{for all atoms $p$, $V^-(t, p)=V_1^-(t, p)$, $V^+(t, p)=V_1^+(t, p)$,}\\
  & \hspace{2.8cm}\text{$V_1^-(h, p)\le V^-(h, p)$, $1-V_1^+(h, p)\le 1-V^+(h, p)$, and}\\
  & \text{there exists an atom $a$ such that $V_1^-(h, a) < V^-(h, a)$ or} \\   
  & \hspace{2.8cm}\text{$nneg(V_1)^-(h, na) < nneg(V)^-(h, na)$}  \\
\text{iff~~~} 
  & \text{for all atoms $p$ and $np$, $nneg(V)(t, p) = nneg(V_1)(t, p)$,}           \\
  & \hspace{2.8cm}\text{$nneg(V)(t, np) = nneg(V_1)(t, np)$,} \\
  & \hspace{2.8cm}\text{$nneg(V)(h, p) \subseteq nneg(V_1)(h, p)$,} \\
  & \hspace{2.8cm}\text{$nneg(V)(h, np) \subseteq nneg(V_1)(h, np)$, and} \\
  & \text{there exists an atom $a$ or $na$ such that $nneg(V)(h, a) \subset nneg(V_1)(h, a)$ or } \\
  & \hspace{2.8cm}\text{$nneg(V)(h, na) \subset nneg(V_1)(h, na)$}  \\
\text{iff~~~} &nneg(V_1) \prec nneg(V).
\end{align*}
\qed
\end{proof}

\noindent{\bf Proposition~\ref{lem:eqmodel_strong_neg}.}\ \  
{\sl 
For any fuzzy formula $F$ that may contain strong negation, a valuation $V$ is an equilibrium model of $F$ iff $nneg(V)$ is an equilibrium model of $nneg(F)$.
}\medskip

\begin{proof}
Let $\sigma$ be the underlying signature of $F$, and let $\sigma'$ be the extended signature $\sigma \cup \left\{np \mid p \in \sigma\right\}$.

\medskip\noindent
($\Rightarrow$) Suppose $V$ is an equilibrium model of $F$. Then 
$V(h, p)=V(t, p)$ for all atoms $p\in \sigma$ and $V$ is a model of $F$. By Lemma \ref{lem:ht_eq}, $nneg(V)(h, a)=nneg(V)(t, a)$ for all atoms $a\in\sigma'$, and by Corollary~\ref{cor:model_strong_neg}, $nneg(V)$ is a model of $nneg(F)$. 

Next we show that there is no $V_1 \prec nneg(V)$ such that $V_1$ is a model of $nneg(F)$. Suppose, for the sake of contradiction, that there exists such $V_1$. We construct $V'$ as 
\[
   V'(w,p) = [V_1^-(w,p), 1-V_1^-(w,np)]
\]
for all atoms $p\in\sigma$. It is clear that $nneg(V')=V_1$.
We will check that
\bi
\ii[(i)] $V'$ is a valid N5 valuation,
\ii[(ii)] $V'\prec V$, and
\ii[(iii)] $V'$ is a model of $F$.
\ei
These claims contradicts the assumption that $V$ is an equilibrium model of $F$. Consequently, we conclude that $nneg(V)$ is an equilibrium model of $F'$ once we establish these claims.

To check (i), we need to show that $0\le V_1^-(w,p)\le 1-V_1^-(w,np)\le 1$ and $V'(t,p)\subseteq V'(h,p)$ for all atoms $p\in \sigma$. Together they are equivalent to checking 
\[
  0\le V_1^-(h,p) \le V_1^-(t,p) \le 1-V_1^-(t,np) \le 1- V_1^-(h,np)\le 1
\] 
The first and the last inequalities are obvious. The second and the fourth inequalities are clear from the fact that $V_1$ is a valid valuation.
To show the third inequality, since $V$ is a valid valuation, 
\[
  V^-(t,p)\le V^+(t,p)
\]
which is equivalent to 
\beq
  nneg(V)^-(t,p) \le 1-nneg(V)^-(t,np).
\eeq{nneg1}
Since $V_1 \prec nneg(V)$, we have $V_1(t, a)=nneg(V)(t, a)$  for all atoms $a \in \sigma'$, so \eqref{nneg1} is equivalent to 
\[
  V_1^-(t,p) \le 1-V_1^-(t,np).
\]
So (i) is verified.

To check claim (ii), notice $nneg(V') \prec nneg(V)$ since $V_1\prec nneg(V)$ and $V_1 = nneg(V')$. Then the claim follows by Lemma~\ref{lem:prec_eq}.

To check claim (iii), notice $nneg(V')$ is a model of $nneg(F)$ since $V_1$ is a model of $nneg(F)$, and $V_1=nneg(V')$. Then the claim follows by Corollary~\ref{cor:model_strong_neg}.

\medskip\noindent
($\Leftarrow$) Suppose $nneg(V)$ is an equilibrium model of $nneg(F)$. Then $nneg(V)(h, a)=nneg(V)(t, a)$ for all atoms $a \in \sigma'$ and $nneg(V)$ is a model of $nneg(F)$. By Lemma \ref{lem:ht_eq}, $V(h, a)=V(t, a)$ for all atoms $a \in \sigma'$ and by Corollary~\ref{cor:model_strong_neg}, $V$ is a model of $F$. 
Next we show that there is no $V_1\prec V$ such that $V_1$ is a model of $F$. Suppose, for the sake of contradiction, that there exists such $V_1$. Then by Lemma \ref{lem:prec_eq}, $nneg(V_1) \prec nneg(V)$ and by Corollary~\ref{cor:model_strong_neg}, $nneg(V_1)$ is a model of $nneg(F)$, which contradicts that $nneg(V)$ is an equilibrium model of $nneg(F)$. 
Consequently, we conclude that $V$ is an equilibrium model of $F$. 
\qed
\end{proof}

\bigskip\noindent{\bf Theorem~\ref{thm:equil-sm}.\optional{thm:equil-sm}}\ \ 
{\sl
For any fuzzy formula $F$ of signature $\sigma$ that may contain strong negation, 
\begin{itemize}
\item[(a)]  A valuation $V$ of $\sigma$ is a fuzzy equilibrium model
  of $F$ iff $V(h,p)=V(t,p)$ for all atoms $p$ in~$\sigma$ and $\I_{nneg(V)}$ is a
  stable model of $\i{nneg}(F)$ relative to $\sigma\cup\{np \mid
  p\in\sigma\}$.

\item[(b)] An interpretation $I$ of $\sigma\cup\{np \mid p\in\sigma\}$ is
  a stable model of $\i{nneg}(F)$ relative to $\sigma\cup\{np \mid
  p\in\sigma\}$ iff $I=\I_{nneg(V)}$ for some fuzzy equilibrium model $V$ of
  $F$. 
\end{itemize}
}
\medskip

\begin{proof}
(a) ($\Rightarrow$) Suppose $V$ is an equilibrium model of $F$. By the definition of an equilibrium model, $V(h, a) = V(t, a)$ for all atoms $a$. By Proposition~\ref{lem:eqmodel_strong_neg}, $nneg(V)$ is an equilibrium model of $nneg(F)$. By Theorem~\ref{thm:equil-sm-nostrneg}, $\I_{nneg(V)}$ in the sense of Theorem~\ref{thm:equil-sm-nostrneg} is a stable model of $nneg(F)$ relative to $\sigma\cup\{np \mid p\in\sigma\}$.

\medskip\noindent
($\Leftarrow$) Suppose $\I_{nneg(V)}$ in the sense of Theorem~\ref{thm:equil-sm-nostrneg} is a stable model of $\i{nneg}(F)$ relative to $\sigma\cup\{np \mid p\in\sigma\}$. By Theorem~\ref{thm:equil-sm-nostrneg}, $nneg(V)$ is an equilibrium model of $\i{nneg}(F)$. By Proposition~\ref{lem:eqmodel_strong_neg}, $V$ is an equilibrium model of $F$.

\bigskip\noindent
(b) ($\Rightarrow$) Suppose $I$ is a stable model of $\i{nneg}(F)$ relative to $\sigma\cup\{np \mid p\in\sigma\}$. Then 
$I \models \underset{p\in \sigma}{\fand_m}\neg_s(p \fand_l np)$. It follows that, for all atoms $p \in \sigma$, $\u_I(p)+\u_I(np) \leq 1$ and thus $\u_I(p) \leq 1 - \u_I(np)$. Construct the valuation $V$ of $\sigma$ by defining $V(w, p)=[\u_I(p), 1-\u_I(np)]$. Clearly, $I$ can be viewed as $\I_{nneg(V)}$.
By Theorem~\ref{thm:equil-sm-nostrneg}, $nneg(V)$ is an equilibrium model of $nneg(F)$. By Proposition~\ref{lem:eqmodel_strong_neg}, $V$ is an equilibrium model of~$F$.

\medskip\noindent
($\Leftarrow$) Take any fuzzy equilibrium model $V$ of $F$. 
By Proposition~\ref{lem:eqmodel_strong_neg}, $nneg(V)$ is an equilibrium model of $nneg(F)$. By Theorem~\ref{thm:equil-sm-nostrneg}, $\I_{nneg(V)}$ is a stable model of $nneg(F)$. 
\qed
\end{proof}

\subsection{Proof of Theorem \ref{thm:constraint}}

\noindent{\bf Theorem~\ref{thm:constraint}.\optional{thm:constraint}}\ \ 
{\sl
For any fuzzy formulas $F$ and $G$, $I$ is a stable model of \hbox{$F\fand\neg G$} (relative to ${\bf p}$) if and only if $I$ is a stable model of $F$ (relative to ${\bf p}$) and $I\models\neg G$.
}
\medskip

\begin{proof}

\noindent
($\Rightarrow$) Suppose $I$ is a $1$-stable model of $F\fand \neg G$ relative to ${\bf p}$. Then $I \models_{1} F\fand \neg G$ and there is no $J <^{\bf p} I$ such that $J \models_{1} F^{{I}}\fand (\neg G)^{I}$. By Lemma~\ref{cor_tnorm_conjunction}, $I\models_{1} F$ and $I\models_{1} \neg G$. 
Note that $\u_I(\neg G)=1$.
Further, it can be seen that there is no $J <^{\bf p} I$ such that $J \models_{1} F^{{I}}$, since otherwise this $J$ must satisfy $J <^{\bf p} I$ and $J \models_{1} F^{{I}}\fand \u_I(\neg G)$ (i.e., $J\models_1 F^I\fand(\neg G)^I$), which contradicts that $I$ is $1$-stable model of $F\fand\neg G$ relative to ${\bf p}$. We conclude that $I$ is a $1$-stable model of $F$ relative to ${\bf p}$ and $I \models_{1} \neg G$.

\medskip\noindent
($\Leftarrow$) Suppose $I$ is a $1$-stable model of $F$ relative to ${\bf p}$ and $I \models_{1} \neg G$. Then $I \models_1 F$ and $I \models_1 \neg G,$. 
By Lemma~\ref{cor_tnorm_conjunction}, $I \models_1 F \fand \neg G$.\footnote{%
This does not hold if the threshold considered is not $1$. For example, suppose $\u_I(F)=0.5$ and $\u_I(G)=0.5$, and consider $\fand_l$ as the fuzzy conjunction. Clearly, $I\models_{0.5} F$ and $I\models_{0.5} G$ but $I \not\models_{0.5} F \fand_l G$.} 
Next we show that there is no $J<^{\bf p} I$ such that $J \models_1 (F\fand \neg G)^{{I}}=F^{{I}}\fand (\neg G)^{I}$. Suppose, to the contrary, that there exists such $J$. Then by Lemma~\ref{cor_tnorm_conjunction}, $J \models_{1} F^{{I}}$. This, together with the fact $J <^{\bf p} I$, contradicts that $I$ is a $1$-stable model of $F$ relative to ${\bf p}$. So there does not exist such $J$, and we conclude that $I$ is a $1$-stable model of $F\fand \neg G$ relative to ${\bf p}$.
\qed
\end{proof}

\BOCC
\noindent{\bf Theorem~\ref{thm:constraint2}.\optional{thm:constraint2}}\
\ 
{\sl
For any fuzzy formulas $F$ and $G$, if $I$ is a $y$-stable model of $F\fand\neg G$ (relative to ${\bf p}$), then $I$ is a $y$-stable model of $F$ (relative to ${\bf p}$) and $I\models_y\neg G$.
}

\vspace{3 mm}

\begin{proof}
$(F\fand \neg G)^{{I}}=F^{{I}}\fand \neg G$. Suppose $I$ is a $y$-stable model of $F\fand \neg G$ relative to ${\bf p}$. Then $I \models_{y} F\fand \neg G$ and there does not exist $J <^{\bf p} I$ such that $J \models_{y} (F\fand \neg G)^{{I}}=F^{{I}}\fand (\neg G)^{{I}}$. By Lemma \ref{lem_conjunction_th}, $I \models_{y} F$ and $I \models_{y} \neg G$. Further, it can be seen that there does not exist $J <^{\bf p} I$ such that $J \models_{y} F^{{I}}$, since otherwise this $J$ must satisfy $J <^{\bf p} I$ and $J \models_{y} F^{{I}}\fand (\neg G)^{I}$, which is a contradiction. So $I$ is a $y$-stable model of $F$ and $I \models_{y} \neg G$.
\qed
\end{proof}
\EOCC

\subsection{Proof of Theorem \ref{thm:choice}\optional{thm:choice}}

\noindent{\bf Proposition~\ref{lem-choice-tautology}.}\ \ 
{\sl
For any fuzzy interpretation $I$ and any set ${\bf p}$ of fuzzy atoms, $I \models {\bf p}^{\rm ch}$.
}

\begin{proof}
Suppose ${\bf p}=(p_1, \dots, p_n)$.
\begin{align*}
 \u_I({\bf p}^{\rm ch}) &= \u_I(\{p_1\}^{\rm ch}\fand\dots \fand \{p_n\}^{\rm ch}) \\
  &= \u_I((p_1 \for_l \neg_s\, p_1)\fand\dots \fand (p_n \for_l \neg_s\, p_n)) \\
  &= {min(\u_I(p_1) + 1-\u_I(p_1), 1)}\fand \dots \fand{min(\u_I(p_n) + 1-\u_I(p_n), 1)} \\
  &= 1.
\end{align*}
\qed
\end{proof}

\noindent{\bf Theorem~\ref{thm:choice}.\optional{thm:choice}}\ \ 
{\sl
\bi
\item[(a)] For any real number $y\in [0,1]$, if $I$ is a $y$-stable model of $F$ relative to ${\bf
    p}\cup {\bf q}$, then $I$ is a $y$-stable model of $F$ relative to
  ${\bf p}$.
\item[(b)] $I$ is a $1$-stable model of $F$ relative to ${\bf p}$ iff 
   $I$ is a $1$-stable model of $F\fand {\bf q}^{\rm ch}$ relative
   to ${\bf p}\cup {\bf q}$.
\ei
}
\medskip

\begin{proof}
(a) Suppose $I$ is a $y$-stable model of $F$ relative to ${\bf
    p}\cup {\bf q}$. Then clearly $I \models_y F$. 
Next we show that there is no $J <^{\bf p} I$ such that $J\models_y F^I$. Suppose, for the sake of contradiction, that there exists such $J$. 
Since $\u_J(q) = \u_I(q)$ for all $q \in {\bf q}$, 
$J$ must satisfy $J <^{{\bf p}{\bf q}} I$ and $J \models_y F^{{I}}$, which contradicts the assumption that $I$ is a $y$-stable model of $F$ relative to ${\bf p}\cup {\bf q}$. So such $J$ does not exist, and we conclude that $I$ is a $y$-stable model of $F$ relative to ${\bf p}$.
	
\medskip\noindent	
(b) ($\Rightarrow$) Suppose $I$ is a $1$-stable model of $F$ relative to ${\bf p}$. Clearly, $I\models F$. 
By Proposition~\ref{lem-choice-tautology}, $I\models {\bf q}^{\rm ch}$, and by 
Lemma~\ref{cor_tnorm_conjunction},
$I\models F\fand {\bf q}^{\rm ch}$.\footnote{It is not necessary to have the threshold $y=1$ for this to hold. In general, suppose $I\models_y F$. Since $I\models_{1} {\bf p}^{\rm ch}$, by the property of fuzzy conjunction($\fand(x, 1)=x$), $I \models_{y} F\fand {\bf p}^{\rm ch}$.} 

Next we show that there is no $J<^{{\bf p}{\bf q}} I$ such that $J \models (F \fand {\bf q}^{\rm ch})^I$. 
Suppose, for the sake of contradiction, that there exists such $J$. Since 
$J\models (F \fand {\bf q}^{\rm ch})^I$, i.e., 
\[ 
  J\models F^I\fand (q_{1} \for_l (\neg_s\, q_1)^{I}) 
      \fand \dots \fand (q_{n} \for_l (\neg_s\, q_n)^{I}),
\]
by Lemma~\ref{cor_tnorm_conjunction}, it follows that $J \models (q_{k} \for_l \u_I(\neg_s\, q_k))$ for each $k=1,\dots,n$, which means that $\u_J(q_k) \geq \u_I(q_k)$. 
On the other hand, since $J <^{{\bf p}{\bf q}} I$, we have $\u_J(q_k) \leq \u_I(q_k)$. So we conclude $\u_J(q_k) = \u_I(q_k)$.\footnote{This cannot be concluded if we have $J \models_{y} (q_{k} \for_l \u_I(\neg_s\, q_k))$, instead of $J \models_{1} (q_{k} \for_l \u_I(\neg_s\, q_k))$. So this direction of the theorem does not hold for an arbitrary threshold.} 
Together with the assumption that $J<^{{\bf p}{\bf q}} I$, this implies that 
$J<^{\bf p} I$.
From $J\models (F \fand {\bf q}^{\rm ch})^I$, it follows $J \models F^{{I}}$, which contradicts that $I$ is a stable model of $F$ relative to ${\bf p}$. So such $J$ does not exist, and we conclude that $I$ is a stable model of $F$ relative to ${\bf p}\cup {\bf q}$.

\medskip\noindent
($\Leftarrow$) Suppose $I$ is a $1$-stable model of $F \fand {\bf q}^{\rm ch}$ relative to ${\bf p}\cup{\bf q}$. Then $I \models F \fand {\bf q}^{\rm ch}$. 
By Lemma~\ref{cor_tnorm_conjunction}, $I \models F$. 
Next we show that that there is no $J<^{\bf p} I$ such that $J\models F^I$. Suppose, for the sake of contradiction, that there exists such $J$.


First, it is easy to conclude $J <^{{\bf p}{\bf q}} I$ from the fact that 
$J<^{\bf p} I$ since $J$ and $I$ agree on ${\bf q}$.
Second, by Proposition \ref{lem-choice-tautology}, $J \models {\bf q}^{\rm ch}$. 
Since $J$ and $I$ agree on ${\bf q}$, 
it follows that $J\models ({\bf q}^{\rm ch})^I$.
Since $J\models F^I$, by Lemma~\ref{cor_tnorm_conjunction}, 
$J \models F^I\fand ({\bf q}^{\rm ch})^I$, or equivalently, 
$J \models (F\fand {\bf q}^{\rm ch})^I$. This, together with the fact that $J<^{{\bf p}{\bf q}} I$, contradicts that $I$ is a 1-stable model of  $F \fand {\bf q}^{\rm ch}$ relative to ${\bf p}\cup{\bf q}$. So such $J$ does not exist, and we conclude that $I$ is a $1$-stable model of $F$ relative to ${\bf p}$.
\qed
\end{proof}

\end{document}